%% file: main.tex
\documentclass{article}
\usepackage{graphicx}
\usepackage{booktabs}
\usepackage{siunitx}
\usepackage{subcaption}
\usepackage{enumitem}
\setlist{nosep}
\usepackage{wrapfig}
\usepackage{amsmath, amsfonts, amsthm, amssymb, booktabs}
\usepackage{thm-restate}
\usepackage[final]{neurips_2025}
\usepackage[skins,theorems]{tcolorbox}
\usepackage[colorlinks=true,allcolors=blue]{hyperref}

\tcbset{highlight math style={enhanced,
colframe=red,colback=white,arc=0pt,boxrule=1pt,
boxsep=0.1pt,left=0pt,right=0pt,top=0pt,bottom=0pt}}

\DeclareFontFamily{U}{mathx}{}
\DeclareFontShape{U}{mathx}{m}{n}{<-> mathx10}{}
\DeclareSymbolFont{mathx}{U}{mathx}{m}{n}
\DeclareMathAccent{\widecheck}{0}{mathx}{"71}

\definecolor{ballblue}{rgb}{0.13, 0.67, 0.8}
\definecolor{darkred}{RGB}{139,0,0}

\newcommand{\indep}{\perp \!\!\! \perp}

\newtheorem{algorithm}{Algorithm}

\newcounter{daggerfootnote}

\newcommand{\norm}[1]{\left\| #1 \right\|}
\newcommand{\R}{\mathbb{R}}
\newcommand{\X}{\mathbf{X}}

\newcommand{\proj}[1]{\big|_{#1}}
\renewcommand{\O}[1]{O\left(#1\right)}
\newcommand{\Ome}[1]{\Omega\left( {#1} \right)}
\newcommand{\Omep}[1]{\Omega_p\left( {#1} \right)}
\newcommand{\Op}[1]{O_p\left(#1\right)}

\newcommand{\deleted}[1]{}

\newcommand{\blued}[1]{#1}

\input{macros}
\newenvironment{tightcenter}{%
  \setlength\topsep{0pt}
  \setlength\parskip{0pt}
  \begin{center}
}{
  \end{center}
}

\allowdisplaybreaks

\title{Statistical Inference for Gradient Boosting Regression} 

\author{%
  Haimo Fang\textsuperscript{1},  
  Kevin Tan\textsuperscript{2}\thanks{This work was done prior to employment at Amazon.},
  Giles Hooker\textsuperscript{2}\\
  \textsuperscript{1}School of Economics, Fudan University \\
  \textsuperscript{2}Department of Statistics and Data Science, 
  The Wharton School, University of Pennsylvania
}

\begin{document}

\maketitle

\begin{abstract} 
    Gradient boosting is widely popular due to its flexibility and predictive accuracy. However, statistical inference and uncertainty quantification for gradient boosting remain challenging and under-explored. We propose a unified framework for statistical inference in gradient boosting regression. Our framework integrates dropout or parallel training with a recently proposed regularization procedure that allows for a central limit theorem (CLT) for boosting. With these enhancements, we surprisingly find that \textit{increasing} the dropout rate and the number of trees grown in parallel at each iteration substantially enhances signal recovery and overall performance. Our resulting algorithms enjoy similar CLTs, which we use to construct built-in confidence intervals, prediction intervals, and rigorous hypothesis tests for assessing variable importance. Numerical experiments demonstrate that our algorithms perform well, interpolate between regularized boosting and random forests, and confirm the validity of their built-in statistical inference procedures.
\end{abstract}

\vspace{-1ex}
\section{Introduction}
\vspace{-1ex}

Gradient boosting \citep{friedman2000greedy}, particularly through widely used implementations such as XGBoost \citep{chen2016xgboost}, LightGBM \citep{ke2017lightgbm}, and CatBoost \citep{catboost2018}, has become one of the most powerful and widely adopted methods for supervised learning, especially on tabular data. However, this flexibility and predictive accuracy come at a cost. Unlike their base learners -- such as decision trees or linear models -- boosting methods are typically far less interpretable, and uncertainty quantification for their predictions is far from straightforward. While point estimates may 
suffice for ad hoc prediction tasks, they overlook randomness inherent in both the data and the 
algorithm itself. Uncertainty quantification therefore asks a central question: \emph{if a new dataset were collected 
and models retrained, how different would the resulting predictions be?}

In this light, various methods have been proposed for uncertainty quantification in boosting. These include Langevin boosting \citep{tan2023uqboosting, ustimenko2022sglbstochasticgradientlangevin, malinin2021uncertaintygradientboostingensembles}, k-nearest neighbors-based techniques \citep{brophy2022boostinguq}, Gaussian graphical models \citep{chen2022boostinguq}, quantile regression approaches \citep{yin2023quantileextremegradientboosting}, and probabilistic prediction via natural gradients \citep{duan2020ngboostnaturalgradientboosting}. However, most of these methods lack formal theoretical guarantees, with many relying primarily on heuristic justifications.




A principled route to uncertainty quantification is through statistical inference, e.g. via prediction intervals. However, this remains limited. In the Bayesian setting, \cite{ustimenko2023gradientboostingperformsgaussian} show randomized boosting converges to a kernel ridge regression and approximate Gaussian process posterior. Yet, they re-run the entire boosting procedure to generate even a single posterior sample. \cite{malinin2021uncertaintygradientboostingensembles} suggest a virtual ensemble method, but do not provide formal guarantees.

In the frequentist setting, \cite{zhou2022boulevard} propose a regularization scheme yielding a central limit theorem for boosting via convergence to kernel ridge regression. However, they do not provide proper confidence or prediction intervals, and their method recovers at most half the true signal (rescaling is suggested, but this amplifies errors). Despite these limitations, their framework remains the only foundation for frequentist inference in boosting, which we extend here.

\vspace{-1ex}
\paragraph{Other literature. } 
Recent advances for random forests yield principled statistical methods for uncertainty quantification via the construction of valid confidence intervals \citep{wager2014rfci, wager2017estimationinferenceheterogeneoustreatment, athey2018generalizedrandomforests, mentch2016randomforestci}, and interpretability via hypothesis tests for variable importance \citep{mentch2016formalhypothesistestsadditive, mentch2022rfhyptest}. 
Much of this is achieved through viewing a random forest as an adaptive kernel method \citep{friedberg2021locallinearforests, athey2018generalizedrandomforests}. 
This also holds for various incarnations of boosting \citep{zhou2022boulevard, ustimenko2023gradientboostingperformsgaussian}. 
Despite these findings, the literature on uncertainty quantification and statistical inference for boosting is far sparser -- a gap we fill by exploiting this equivalence.


Motivated by this gap, we study frequentist inference for uncertainty quantification in boosting under supervised nonparametric regression: \(y = \f(\bx) + \epsilon\), with subgaussian noise $\sigma^2$. Our goal is to develop principled, statistically sound procedures for uncertainty quantification, outlined below.
\vspace{-2.5ex}
\paragraph{Methodological contributions.} In detail, we provide the following methodological contributions:
\vspace{-3ex}
\begin{enumerate}[leftmargin=0.5cm]
    \item \emph{Improvements via dropout:} We incorporate random dropout \citep{rashmi2015dartdropoutsmeetmultiple} into the regularization procedure of \cite{zhou2022boulevard}. This yields an improved procedure (Algorithm \ref{alg:random-dropout}) that achieves provably better performance by up to a factor of $4$ in the asymptotic relative efficiency (ARE), and increased signal recovery by up to a factor of $2$. By tuning the dropout rate in Algorithm \ref{alg:random-dropout}, we smoothly bridge the vanilla Boulevard method \citep{zhou2022boulevard} and a modified Random Forest \citep{breiman2001random} that draws observation subsamples without replacement and considers the full covariate space at each split.
    \item \emph{Parallel boosting:} We construct a novel leave-one-out procedure for parallel boosting (Algorithm \ref{alg:structured-dropout}) that, when coupled with the above regularization procedure, improves in the ARE by at least a factor of $4$, and enjoys increased signal recovery by at least a factor of $2$.
    \item \emph{Prediction, confidence, and reproduction intervals:} We leverage the central limit theorems our procedures enjoy to construct a variety of procedures for statistical inference. 
    First, we conduct predictive inference via constructing prediction intervals on (potentially unseen) labels $\by$. We then construct confidence intervals for the underlying ground truth function $\f$. Lastly, to quantify uncertainty within the learning algorithm, we provide reproduction intervals \citep{zhou2022boulevard} for another booster $\widehat{\f}$ trained on an independent realization of the training set. 
    \item \emph{Hypothesis testing:} We construct a chi-squared hypothesis test for variable importance, extending the work of \cite{mentch2016randomforestci} for random forests to the boosting setting. 
    \item \emph{Computational efficiency:} Unlike prior approaches to statistical inference for random forests and boosting, we utilize the Nystr\"{o}m approximation of \cite{musco2017recursivenystrom}. This allows our tests and intervals to scale linearly with the number of data points.
\end{enumerate}
\vspace{-2ex}
\begin{figure}[H]
    \centering
    \includegraphics[width=0.9\linewidth]{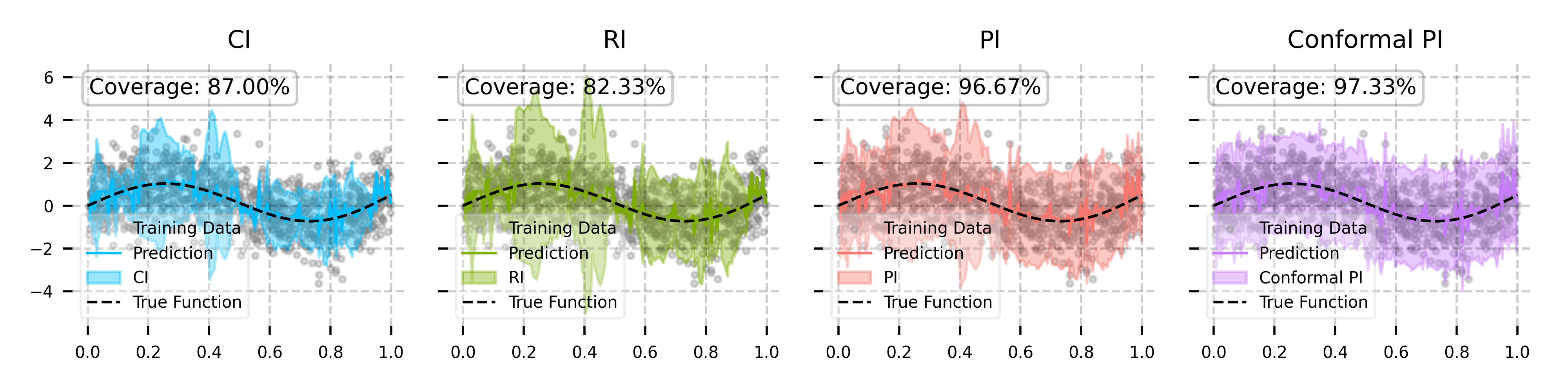}
    \vspace{-2ex}
    \caption{Demonstration of confidence, prediction, and reproduction intervals on $\f(x)=\sin 2\pi x + \frac{1}{2} x^2$. Conformal baseline. $200$ trees, learning rate $0.6$, depth $8$, subsampling and dropout $0.6$.}
    \label{fig:1d-coverage}
    \vspace{-2ex}
\end{figure}

\vspace{-1ex}
\textbf{Theoretical contributions.} Our primary theoretical contributions are central limit theorems (CLTs) for boosting with dropout (Algorithm \ref{alg:random-dropout}) and parallel training (Algorithm \ref{alg:structured-dropout}) -- the first such results for these settings to our knowledge. These extend the theoretical arguments of \cite{zhou2022boulevard}. In analyzing boosting with dropout, the stochasticity introduced into ensemble construction must be accounted for -- in addition to randomness from data subsampling. To address this, we establish a weak law for finite-sample convergence in the stochastic contraction framework of \cite{almudevar2022stochasticcontractionmappingtheorem}. 
The arguments for parallel boosting are more challenging. Inspired by classical backfitting methods \citep{breiman1985backfitting}, we develop a delayed averaging scheme that enables parallelism while preserving convergence guarantees. To achieve central limit theorems that hold unconditionally on the data $\bX$, we show a bound on the maximal leaf size is sufficient to limit the influence of distant points and promote balanced splits, enabling rigorous distributional results.

\textbf{Numerical experiments.} We establish the efficacy of our methods on both a simulation study and real-world datasets. Our numerical experiments showcase that our algorithms can be tuned to interpolate between regularized boosting and random forests, and are highly competitive in terms of MSE. These also demonstrate the correctness, coverage, and computational efficiency of the statistical procedures we construct via the central limit theorems our algorithms enjoy.


\vspace{-1ex}
\section{Setup}
\label{sec:setup}
\vspace{-1ex}

Consider boosting for nonparametric regression with squared error loss. Given random features $\bx \in [0,1]^d$, labels $y \in \R$, and noise $\epsilon \sim \text{SubG}(0,\sigma^2)$, say there exists some function $\f$ such that $y = \f(\bx) + \epsilon$. The learner is given a training set $(\bX^{\train}, \by^{\train})$ of size $n$, and learns an estimate $\widehat{\f}$ of $\f$ by minimizing the mean squared error (MSE) $\frac{1}{n} \sum_{i=1}^n (\widehat{\f}(\bx_i) - y_i)^2$ over the training set. It is customary to analyze how $\widehat{\f}$ converges to $\f$, or in other words, the generalization error of $\widehat{\f}$ on a test set $(\bX^{\test}, \by^{\text{test}})$. Subscripts $\widehat{\f}_n^{(b)}$ indicate the estimated predictor trained on $n$ datapoints for $b$ boosting rounds, where we omit dependencies on $b$ and $n$ whenever possible.

\vspace{-1ex}
\paragraph{Regularized stochastic gradient boosting.} \cite{zhou2022boulevard} propose a regularization procedure called Boulevard that recovers a central limit theorem for boosting. They do so by showing that the resulting regularized procedure converges to a kernel ridge regression. At each iteration
\begin{tightcenter}
    $\textstyle{b=1,...,B-1, \text{ instead of } \widehat{\f}^{(b+1)} \gets \widehat{\f}^{(b)} + \lambda \t^{(b)}, \text{ they update } \widehat{\f}^{(b+1)} \gets \frac{b-1}{b} \widehat{\f}^{(b)} + \frac{\lambda}{b} \t^{(b)}}$
\end{tightcenter}
\vspace{-1.2ex}
for the current function estimate $\widehat{\f}^{(b)}$, built tree $\t^{(b)}$, and learning rate $\lambda \in [0,1)$. As this is equivalent to updating $\widehat{\f}^{(b+1)} \gets \frac{\lambda}{b} \sum_{i=1}^b \t^{(i)}$, the resulting ensemble is an average of trees. This observation, along with assumptions on tree adaptivity, allows them to prove that Boulevard predictions are asymptotically normal, though they do not construct confidence or prediction intervals. Additionally, this regularization comes at a cost. It turns out that $\widehat{\f}_n^{(b)}$ converges only to $\frac{\lambda}{1+\lambda} \f \leq \f / 2$ as $n,b \to \infty$, and so \emph{Boulevard recovers at most half the signal}. Still, this is the only method available for frequentist inference for gradient boosting, and so we seek to improve upon it in this paper.
\vspace{-2ex}
\begin{figure}[H]
\centering
    \includegraphics[width=0.49\linewidth]{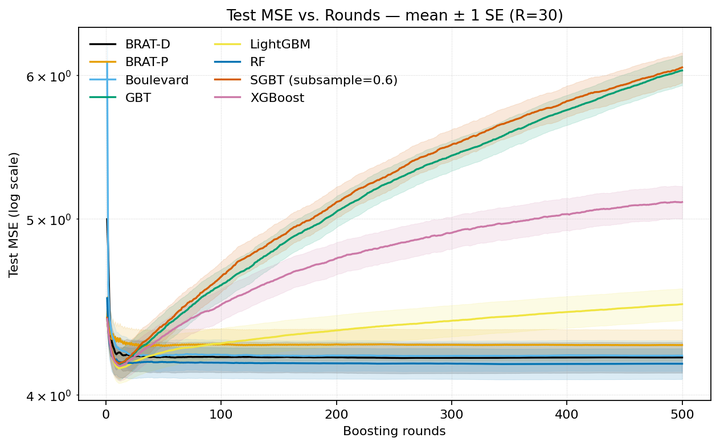}\includegraphics[width=0.32\linewidth]{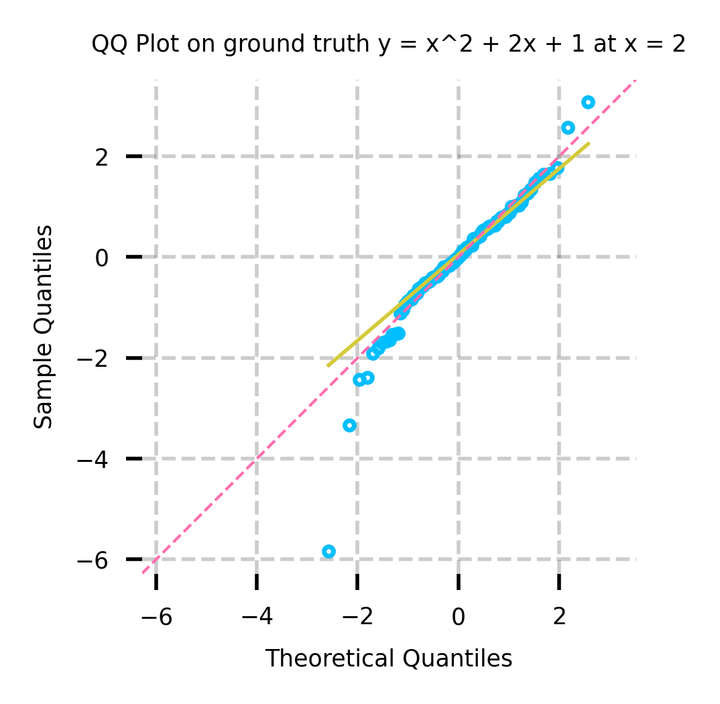}
    \vspace{-1ex}
    \caption{An added benefit of Boulevard regularization is that it renders early stopping unnecessary. Because the procedure converges to a fixed point, one can (in theory) boost indefinitely without encountering overfitting. It also produces asymptotically normal predictions (see the Q-Q plot).}
    \label{fig:no-need-early-stopping}
\end{figure}
\vspace{-2ex}

\vspace{-2ex}
\paragraph{Regression trees and tree structure.} We are primarily concerned with boosting algorithms where we use regression trees as weak learners. We introduce a formulation of a regression tree below.
\begin{defn}[Regression trees]
A regression tree $\t_n$ trained on $n$ datapoints segments the covariate space $[0,1]^d$ into a partition of hyper-rectangles $\{A_i\}_{i=1}^m$. When $A(\bx)$ for the rectangle containing $\bx$, and $s_{n,j}(\bx)$ is the frequency at which training point $\bx_j$ and test point $\bx$ share a leaf: 
\begin{tightcenter}
    $\textstyle{s_{n,j}(\bx) = \frac{\mathbbm{1}(\bx_j \in A(\bx))}{\sum_{k=1}^n \mathbbm{1}(\bx_k \in A(\bx))}, \text{ the regression tree } \t_n \text{ predicts }
\t_n(\bx) = \sum_{j=1}^n s_{n,j}(\bx)\cdot y_j}.$
\end{tightcenter}
\label{defn:regression-trees}
\end{defn}
\vspace{-1ex}
As such, a tree can be thought of as an adaptive nearest neighbor method \citep{athey2018generalizedrandomforests, friedberg2021locallinearforests}, where $s_{n,j}(\bx)$ yields the influence of datapoint $\bx_j$ on the tree's prediction at test point $\bx$. This motivates the definition of a structure vector and matrix from \cite{zhou2022boulevard}:
\vspace{-2ex}
\begin{defn}[Structure vectors and matrices]
    Let $\t_n$ be a tree trained on the input $(\bX_n, \by_n)$. For an arbitrary point $\bx$, let $\s_n(\bx) = (s_{n,1}(\bx),...,s_{n,n}(\bx))^\top$ be the structure vector of $\bx$. As such, we say that the matrix $\bS_n = \left(s_{n,j}(\bx_i)\right)_{i,j=1}^n$ with rows $\s_n(\bx_i)^\top$ is the structure matrix of $\t_n$.
    \label{defn:structure-matrix}
\end{defn}
\vspace{-1ex}
The structure matrix is a kernel matrix. A natural idea is to perform kernel ridge regression (KRR):
$$\widehat{\f}_n(\bx) = \langle\r_n(\bx), \by_n\rangle, \text{ } \k_n(\bx) = \E[\s_n(\bx)], \; \bK_n = \E[\bS_n], \text{ and } \r_n(\bx)^\top = \k_n(\bx)^\top (\lambda^{-1}\bI + \bK_n)^{-1}.$$
\cite{zhou2022boulevard} show Boulevard regularization converges to a kernel ridge regression $\widehat{\f}_n(\bx) = \langle \r_n(\bx), \by_n\rangle$ as $b \to \infty$. However, as $\lambda \in (0,1]$, this converges to at most half the signal. We later propose two algorithms in Algorithms \ref{alg:random-dropout} and \ref{alg:structured-dropout}, that converge to $\langle\r_n^D(\bx), \by_n\rangle = \k_n(\bx)^\top (\lambda^{-1}\bI + q\bK_n)^{-1} \by_n$ and $\langle\r_n^P(\bx), \by_n\rangle = \k_n(\bx)^\top (\bI + (K-1)\bK_n)^{-1}K \by_n$ respectively, achieving increased signal recovery. 
As a final note, to accommodate training $\t_n(\bx, \gG)$ on a subsample $\gG \subseteq \{1,...,n\}$, we write $s_{n,j}(\bx; \gG) = {\mathbbm{1}(\bx_j \in A, j \in \gG)}/{\sum_{k=1}^n \mathbbm{1}(\bx_k \in A, k \in \gG)}$, and $\t_n(\bx; \gG) = \sum_{j=1}^n s_{n,j}(\bx;\gG) \cdot y_j$. We suppress the dependence on subsamples when possible.

\vspace{-1ex}
\paragraph{Integrity \& adaptivity.} We inherit the following two assumptions from Boulevard regularization:
\vspace{-2ex}
\begin{aspt}[Structure-value isolation]
\label{aspt:svi}
    $(\t_n^{(b)})_{b=1}^B$ has $\s_n^{(b)}(\bx) \indep \by_n$ for all $b=1,...,B$.
\end{aspt}
\begin{aspt}[Non-adaptivity]
\label{aspt:non-adaptivity}
    There exists a probability measure \(\gQ_n\) on the space of tree structures and some $b'$ so the tree structure at iteration $b=b',...,B$ is an independent draw from \(\gQ_n\).
\end{aspt}
\vspace{-1ex}
These are required for theoretical guarantees. In practice, the first can be implemented via a stronger form of honesty \citep{wager2017estimationinferenceheterogeneoustreatment} that we call \emph{integrity}, refitting leaf values on an independent dataset after boosting. One can also randomly sample splits, or get them from a random forest trained on an independent sample. The second only requires non-adaptivity after a finite number of iterations, which may occur naturally \citep{zhou2022boulevard}. To enforce it, one can boost normally for a fixed number of steps and then sample tree structures from earlier iterations. Still, numerical experiments indicate that our algorithms perform well even without these adjustments.

\vspace{-1ex}
\section{Algorithms}
\label{sec:algorithms}
\vspace{-1ex}

To address the limitations of Boulevard regularization while retaining its statistical guarantees, we propose two improvements that achieve enhanced signal recovery and improved sample complexity:
\begin{enumerate}[label=\textbf{Alg. \arabic*:}]
    \item By incorporating random dropout \citep{rashmi2015dartdropoutsmeetmultiple} in the ensemble.
    \item By growing trees in parallel within each round with a leave-one-out procedure.
\end{enumerate}
The intuition is as follows. Surprisingly, increasing the dropout probability allows one to recover more of the signal. This is because each tree in the current ensemble is effectively downweighted during residual construction, and the new tree is fit on more of the signal. On the other hand, as the number of trees grown in parallel within each round grows, the resulting ensemble increasingly resembles the procedure of boosting random forests in sequence -- allowing each round to recover more signal than a single tree can. We present both algorithms and explain them in detail below.

\vspace{-2ex}
\begin{minipage}{0.499\textwidth}
\begin{algorithm}[H]
    \centering
    \caption{Boulevard Regularized Additive regression Trees -- Dropout (BRAT-D)}
    \label{algorithm}
    \begin{algorithmic}[1]
        \STATE \textbf{Input:} Features $\bX$, labels $\by$, dropout rate $p = 1-q$, subsample rate $\xi$, learning rate $\lambda$, boosting rounds $B$. Set $\widehat{\f}^{(0)}, \t^{(0)} \gets 0$. 
        \FOR{$b=1,...,B-1$}
        \STATE Subsample $\gS_b \subseteq \{0,...,b-1\}$ w.p. $q$.
        \STATE Subsample data $\gG_b \subseteq \{1,...,n\}$ w.p. $ \xi$.
        \STATE
        $z_i \gets y_i - \frac{\lambda}{b}\sum_{s \in \gS_b}\t^{(s)}(\bx_i)$.
        \STATE Construct tree $\t^{(b)}$ on $(\bx_i, z_i)_{i \in \gG_b}$.
        \STATE $\widehat{\f}^{(b+1)} \gets \frac{b-1}{b}\widehat{\f}^{(b)} + \frac{\lambda}{b}\t^{(b)} = \frac{\lambda}{b}\sum_{i=1}^b \t^{(i)}$.
        \ENDFOR
        \STATE \textbf{Return} final predictor $\frac{1+\lambda q}{\lambda} \widehat{\f}^{(B)}$.
    \end{algorithmic}
    \label{alg:random-dropout}
\end{algorithm}
\end{minipage}
\hfill
\begin{minipage}{0.49\textwidth}
\begin{algorithm}[H]
    \centering
    \caption{BRAT - Parallel (BRAT-P)}
    \label{algorithm1}
    \begin{algorithmic}[1]
        \STATE \textbf{Input:} Features $\bX$, labels $\by$, subsample $\xi$, trees/round $K$, rounds $B$. $\widehat{\f}^{(0)}, \t^{(0,k)} \gets 0$. 
        \STATE Fit $\widehat{\f}^{(1)}, (\t^{(1,k)})_{k=1}^K$ with $K$ boosting steps. 
        \FOR{$b=2,...,B-1$}
        \FOR{$k=1,...,K$ \textbf{in parallel}}
        \STATE Subsample $\gG_{b,k} \subseteq \{1,...,n\}$ w.p. $ \xi$.
        \STATE \hspace{-1.1mm} $z_{i,k} \gets y_i - \frac{1}{b-1}\sum_{s=1}^{b-1}\sum_{l\neq k}\t^{(s,l)}(\bx_i)$
        \STATE Construct tree $\t^{(b,k)}$ on $(\bx_i, z_{i,k})_{i \in \gG_{b,k}}$.
        \ENDFOR
        \STATE Define $\widehat{\f}^{(b+1)} \gets \frac{1}{b}\sum_{s=1}^{b}\sum_{k=1}^K \t^{(s,k)}$.
        \ENDFOR
        \STATE \textbf{Return} final predictor $\widehat{\f}^{(B)}$.
    \end{algorithmic}
    \label{alg:structured-dropout}
\end{algorithm}
\end{minipage}


\vspace{-1ex}
\paragraph{Random dropout.} Unlike vanilla boosting, Algorithm \ref{alg:random-dropout} computes residuals with both data \citep{friedman2002stochasticboosting} and ensemble subsampling \citep{rashmi2015dartdropoutsmeetmultiple}. At each round $b=1,...,B-1$, each tree is subsampled with a dropout rate of $p \in [0,1)$ (equivalently, with probability $q=1-p$). We write $\gS_b \subseteq \{0,...,b-1\}$ for the collection of subsampled tree indices at each round $b$. Likewise, data indices $\gG_b \subseteq \{1,...,n\}$ are subsampled with probability $\xi \in (0,1]$. 

As a result, each tree $\t^{(b)}$ is trained on the dataset $(\bx_i, z_i)_{i \in \gG_b}$, where the residuals are given by $z_i = y_i - \frac{1}{b} \sum_{s \in \gS} \t^{(s)}(\bx_i)$. Within the residual computation, we intentionally divide the sum of the subsampled ensemble's predictions by $b$ instead of $|\gS|$. This is done so each new tree is fit on more of the signal, allowing for increased signal recovery when applying Boulevard regularization. 

We later show in Section \ref{sec:theory} that the predictions of Algorithm \ref{alg:random-dropout} converge to a kernel ridge regression that enjoys a central limit theorem. Informally, for any test point $\bx$, we have that:
$$\textstyle{\widehat{\by}^{(b)}_n \underset{b\to\infty}{\overset{a.s.}{\longrightarrow}} \left(\lambda^{-1} \bI+q \mathbb{E}\left[\bS_n\right]\right)^{-1} \mathbb{E}\left[\bS_n\right]\by_n, \quad\lVert\r_n^D(\bx)\rVert_2^{-1}\left(\widehat{\f}_n^D(\bx) - \frac{\lambda}{1+\lambda q}\f(\bx)\right) \underset{n\to\infty}{\overset{d}{\longrightarrow}} \gN(0, \sigma^2)},$$
for a kernel matrix induced by the tree ensemble $\bK_n = \E[\bS_n]$, and limit $\widehat{\f}_{n}^D = \lim_{b \to \infty} \widehat{\f}^{(b)}_{n}$ where $\widehat{\f}_n^D(\bx) = \langle\r_n^D(\bx),  \by_n\rangle = \k_n(\bx)^\top (\lambda^{-1}\bI + q\bK_n)^{-1}\by_n$. $\lVert\r_n\rVert_2$ is the norm of kernel weights. Algorithm \ref{alg:random-dropout} encompasses two important special cases. When $\lambda=1$ and $p\to 1$, this yields a random forest. When $p=0$, Algorithm \ref{alg:random-dropout} corresponds to the original Boulevard algorithm \citep{zhou2022boulevard}. Tuning the dropout parameter $p$ allows the user to interpolate between the two, while improving over \cite{zhou2022boulevard} by up to a factor of $\frac{1+\lambda}{1+\lambda q} \in (1, 2]$ in signal recovery.

\vspace{-2ex}
\paragraph{Parallel boosting.} Algorithm \ref{alg:structured-dropout} integrates a novel leave-one-out procedure and Boulevard regularization into the parallel boosting framework \citep{long2011parallelmargin, karbasi2023impossibilityparallelizingboosting, dacunha2024optimalparallelizationboosting}. The algorithm is warm-started by performing $K$ vanilla boosting iterations in the first round. At each subsequent round $b=2,...,B-1$, we grow $K$ trees $(\t^{(b,k)})_{k=1}^K$ in parallel on the $K$ datasets $((\bx_i, z_{i,k})_{i \in \gG_{b,k}})_{k=1}^K$. These consist of independently subsampled data indices $\gG_{b,k} \subseteq \{1,...,n\}$, and residuals that are computed by leaving one ``column'' out (or alternatively, one tree out per round): $z_{i,k} = y_i - \frac{1}{b-1}\sum_{s=1}^{b-1}\sum_{l\neq k}\t^{(s,l)}(\bx_i)$. The resulting predictor is given by the regularized boosting update: $\widehat{\f}^{(b+1)} \gets \frac{b-1}{b} \widehat{\f}^{(b)} + \frac{1}{b} \sum_{k=1}^K \t^{(b,k)} = \frac{1}{b}\sum_{s=1}^{b}\sum_{k=1}^K \t^{(s,k)}$.

Likewise, we later show in Section \ref{sec:theory} that the predictions of Algorithm \ref{alg:structured-dropout} also converge to a kernel ridge regression that enjoys a central limit theorem. Informally, for any test point $\bx$:
$${\widehat{\by}^{(b)}_n \underset{b\to\infty}{\overset{a.s.}{\longrightarrow}} \left(\bI+ (K-1)\mathbb{E}\left[\bS_n\right]\right)^{-1} K\mathbb{E}\left[\bS_n\right]\by_n, \quad \lVert\r_n^P(\bx)\rVert_2^{-1}\left(\widehat{\f}_n^P(\bx) - \f(\bx)\right) \underset{n\to\infty}{\overset{d}{\longrightarrow}} \gN(0, \sigma^2)},$$
with limit $\widehat{\f}_{n}^P = \lim_{b \to \infty} \widehat{\f}^{(b)}_{n}$ given by $\widehat{\f}_{n}^P(\bx)= \langle\r_n^P(\bx), \by_n\rangle = \k_n(\bx)^\top (\bI + (K-1)\bK_n)^{-1}K \by_n$. 
Algorithm \ref{alg:structured-dropout} can be viewed as boosting a modified parallel backfitting algorithm \citep{breiman1985backfitting}\footnote{This can be thought of as a \emph{structured} variant of dropout \citep{zhao2022revisitingstructureddropout, xin2020deebertdynamicearlyexitingdropout, fan2019reducingtransformerdepthdemanddropout}. Dropping one out of every $K$ trees allows the regularization to scale in $1/K$.} that is warm-started with $K$ vanilla boosting iterations.\footnote{The initial $K$ boosting iterations yield better convergence in practice, and do not hinder our theoretical arguments as the impact of the vanilla boosting iterations washes out when we take $B \to \infty$.} When $K=1, B\to\infty$, Algorithm \ref{alg:structured-dropout} reduces to a random forest $\mathbb{E}\left[\bS_n\right]\by_n$. On the other hand, Algorithm \ref{alg:structured-dropout} reduces to vanilla boosting when $B=1, K\to\infty$. Note that the CLT only holds for $1 < K < \infty$. 

The attentive reader may observe that we do not divide the final predictions by $K$. This is because the combination of leaving one column out when fitting a tree and Boulevard averaging is sufficient to allow our procedure to stabilize.\footnote{Parallelized backfitting can be thought of as parallelized coordinate descent. Averaging over the coordinate updates is sufficient for convergence. Dividing by $b$, as in Boulevard averaging, is sufficient when $b \geq K$.} Still, dividing by $K$ yields a viable algorithm (deferred to future work) that can be thought of as boosting random forests each of $K$ trees over $B$ rounds, using the natural strategy of averaging all trees but the left-out column to form predictions. 




\vspace{-1ex}
\section{Statistical Inference and Uncertainty Quantification for Boosting}
\label{sec:statistics}
\vspace{-1ex}

We are now in a place to introduce our methods for conducting statistical inference for boosting; theoretical basis for these procedures is provided below. Our methods rest on the aforementioned convergence to a kernel ridge regression and central limit theorems that we formally prove later in Theorems \ref{thm:krrconv} and \ref{thm:main}. In the meantime, we provide the following asymptotically valid procedures.

\vspace{-1ex}
\paragraph{Estimation of asymptotic variance.} The CLTs for Algorithms \ref{alg:random-dropout} and \ref{alg:structured-dropout} imply that one can construct confidence intervals for $\f$ with the following normal approximations:
$$\lambda^{-1}(1+\lambda q)\widehat{\f}_n^D(\bx) \stackrel{d}{\approx} \gN\left(\f(\bx), \lambda^{-2}(1+\lambda q)^2\lVert\r_n^D(\bx)\sigma^2 \rVert_2^2 \right), \;\; \widehat{\f}_n^P(\bx) \stackrel{d}{\approx} \gN\left(\f(\bx), \sigma^2 \lVert\r_n^P(\bx)\rVert_2^2\right).$$
As such, it remains to estimate $\r_n^D, \r_n^P$, and $\sigma^2$. We first estimate $\k_n(\bx)$ by the sample averages
\begin{equation}
    \textstyle{\widehat{\k}_n^D(\bx) = \frac{1}{B} \sum_{b=1}^B s_{n,i}^{(b)}(\bx, \gG_b),\;\; \widehat{\k}_n^P(\bx) = \frac{1}{BK} \sum_{b=1}^B \sum_{k=1}^K s_{n,i}^{(b,k)}(\bx, \gG_{b,k}),}
    \label{eqn:hat-k}
\end{equation}
where $s_{n,i}^{(b)}(\bx; \gG_b)$ is the fraction of subsampled datapoints in the same leaf $A_j^{(b)}$ as $\bx$ over the trees. Likewise, we estimate $(\bK_n)_{i,j}$, the fraction of trees where datapoints $i$ and $j$ fall in the same leaf:
\begin{equation}
    \textstyle{(\widehat{\bK}_n^D)_{i,j} = \frac{1}{B} \sum_{b=1}^B s_{n,j}^{(b)}(\bx_i; \gG_b), \;\;(\widehat{\bK}_n^P)_{i,j} = \frac{1}{B} \sum_{b=1}^B s_{n,j}^{(b,k)}(\bx_i; \gG_{b,k})},
    \label{hat-K}
\end{equation}
Our estimates of $\widehat{\r}_n$ are then given by plugging the above into the formula for $\r_n$:
$$\widehat{\r}_n^D(\bx)^\top = \widehat{\k}^D_n(\bx)^\top (\lambda^{-1}\bI + q\widehat{\bK}^D_n)^{-1}, \;\;\widehat{\r}_n^P(\bx)^\top = \widehat{\k}^P_n(\bx)^\top (\bI + (K-1)\widehat{\bK}^P_n)^{-1}K.$$
Lastly, we estimate $\sigma^2$ with the residuals $\widehat{\sigma}^2 = \frac{1}{n_{\text{cal}}}\sum_{i=1}^{n_{\text{cal}}} (y_i - \widehat{y}_i)^2$ on a hold-out calibration set. We describe a tweak later in Section \ref{sec:coverage-experiments} that uses the calibration set for increased robustness.

\vspace{-1ex}
\paragraph{Confidence intervals for $\f$.} Putting everything together yields $1-\alpha$ confidence intervals for $\f$ (where $z_{1-\alpha/2}$ is the $(1-\alpha/2)$-quantile of the standard normal distribution):
\begin{equation}
    \lambda^{-1}(1+\lambda q)\widehat{\f}_n^D(\bx) \pm z_{1-\alpha/2}\lambda^{-1}(1+\lambda q)\widehat{\sigma}\lVert\widehat{\r}_n^D(\bx)\rVert_2 ,\;\; \widehat{\f}_n^P(\bx) \pm z_{1-\alpha/2} \widehat{\sigma}\lVert\widehat{\r}_n^P(\bx)\rVert_2.
    \label{eqn:confidence-intervals}
\end{equation}

\vspace{-1ex}
\paragraph{Prediction intervals for $\by$.} We can construct prediction intervals for $y | \bx$ as follows:
\begin{equation*}
    \lambda^{-1}(1+\lambda q)\widehat{\f}_n^D(\bx) \pm z_{1-\alpha/2}\lambda^{-1}(1+\lambda q)\widehat{\sigma}\sqrt{1+\lVert\widehat{\r}_n^D(\bx)\rVert_2^2},\;\; \widehat{\f}_n^P(\bx) \pm z_{1-\alpha/2} \widehat{\sigma}\sqrt{1+\lVert\widehat{\r}_n^P(\bx)\rVert_2^2}.
    \label{eqn:prediction-intervals}
\end{equation*}
These prediction intervals have asymptotic pointwise coverage, conditional on the test point $\bx$. That is, $\prob(y \in \text{PI}(\bx) | \bx) \to 1-\alpha$ as $n \to \infty$, which holds as a corollary of the CLTs we present in the following section. Note that this conditional coverage guarantee is a stronger guarantee than is typically possible with conformal inference \citep{barber2020limitsdistributionfreeconditionalpredictive}. We compare our prediction intervals to conformal prediction intervals within our numerical experiments later in this paper.

\vspace{-1ex}
\paragraph{Reproduction intervals for $\widehat{\f}$.} We now turn our attention towards computing reproduction intervals. That is, a confidence interval for another realization of $\widehat{\f}$ trained on another independent realization of the data. To do so, as \cite{zhou2022boulevard} suggest, it suffices to scale the width of the above confidence intervals for $\f$ by $\sqrt{2}$ as $X_1, X_2 \stackrel{i.i.d.}{\sim} \gN(0,\sigma^2) \implies X_1 - X_2 \sim \gN(0,2\sigma^2)$.
Note that one can invert the above confidence, prediction, or reproduction intervals in order to test if the underlying function, another realization of the response, or average learner is equal to a given constant. This amounts to checking if the constant is contained in the interval of choice. 


\vspace{-1ex}
\paragraph{Variable importance tests.} Suppose we wish to test for variable importance. Consider the possibility that there exists some subset of the feature space $[0,1]^d$ containing only $\widecheck{d} < d$ features, and let $\g : [0,1]^{\widecheck{d}} \to \R$ be the projection of $\f$ onto $\gL^2([0,1]^{\widecheck{d}})$. We test the possibility that $\f(\bx) = \g({\bx})$ by proxy. Consider a split of the training dataset $(\bX_n, \by_n)$ into $(\bX_{n,1}, \by_{n,1}), (\bX_{n,2}, \by_{n,2})$, where $\bX_{n,1} \in \R^{n/2, d}, \bX_{n,2} \in \R^{n/2,\widecheck{d}}$. Let $\widehat{\f}_{n,1}$ be the boosting learner trained on $(\bX_n, \by_n)$, and $\widehat{\f}_{n,2}$ be the same trained on $({\bX}_{n,2}, \by_{n,2})$. Given a hold-out dataset $(\bX_m, \by_m)$, we test the null:
$$H_0 : \f(\bx_j) = \g({\bx}_j) \text{ for all } j=1,...,m \text{ against } H_1 : \f(\bx_j) \neq \g({\bx}_j) \text{ for some } j$$
by comparing $\widehat{\f}_{n,1}(\bx_j)$ and $\widehat{\f}_{n,2}({\bx}_j)$. We exploit the CLTs our algorithms enjoy to do so. Write $\widehat{\r}^{D,P}_n(\bX_m)^\top \in \R^{m \times n}$ for the matrix with rows $(\widehat{\r}^{D,P}_n(\bx_l)^\top)_{l=1}^m$, where we write $D,P$ within the superscript if the formula holds for both Algorithms \ref{alg:random-dropout} and \ref{alg:structured-dropout}. The difference in predictions 
$$\widehat{\d}_m = (\widehat{\r}^{D,P}_{n,1}(\bX_m) - \widehat{\r}^{D,P}_{n,2}(\bX_m))^\top \by \; \stackrel{d}{\to} \; \gN_m(\mathbf{0}, \sigma^2 \widehat{\bXi}_n) \text{ is multivariate normal under the null},$$
with covariance matrix $\sigma^2 \widehat{\bXi}_n = \sigma^2 (\widehat{\r}^{D,P}_{n,1}(\bX_m) - \widehat{\r}^{D,P}_{n,2}(\bX_m))^\top (\widehat{\r}^{D,P}_{n,1}(\bX_m) - \widehat{\r}^{D,P}_{n,2}(\bX_m))$ estimated by plugging in $\widehat{\sigma}^2$. As such, we can conduct a chi-squared test with test statistic
$$\widehat{\sigma}^{-2} \widehat{\d}_m^\top \widehat{\bXi}_n^{-1} \widehat{\d}_m \sim \chi_m^2 \text{ under the null}, \text{ rejecting the null if } \widehat{\sigma}^{-2} \widehat{\d}_m^\top \widehat{\bXi}_n^{-1} \widehat{\d}_m > \chi_{m,1-\alpha}^2.$$
The test as presented runs in $O(n^3)$ time. In Appendix \ref{app:random-projections}, we present an accelerated version of the test that runs in 
$O(ns(r+s) + r^3)$ time, where $s$ is the number of points subsampled for Nystr\"{o}m approximation and $r$ is the number of test points subsampled. 



\vspace{-1ex}
\paragraph{Matrix sketching.} Although he CLT reduces the problem of inference to kernel ridge regression (with the right kernel), computing \(\widehat{\bK}_n\) is an $O(n^3)$ problem -- intractable for large $n$. This is the chief difficulty that Zhou and Hooker (2022) face when constructing replication intervals for vanilla Boulevard.
We bypass this issue, making our procedures practical and tractable via matrix sketching.

We approximate \(\widehat{\bK}_n\) using the Nystr\"{o}m method \citep{williams2000nystromkernel}, via either uniform subsampling or the recursive approach of \cite{musco2017recursivenystrom}, producing an approximation \(\widetilde{\bK}_n \approx \widehat{\bK}_n\bS(\bS^\top \widehat{\bK}_n\bS)^{\dagger} \bS^\top \widehat{\bK}_n\), where \(\bS \in \mathbb{R}^{n \times s}\) is a random subsampling matrix. With the Nystrom method of Musco and Musco (2017), we can choose $s$ to near-linear in the effective dimension $d_{\text{eff}}^\mu$ of a ridge regression problem with regularization parameter $\mu$ on the kernel matrix: $s = \widetilde{O}(d_{\text{eff}}^\mu)$.

This requires only $O(ns^2)$ time to precompute the kernel, and only $O(s^2)$ time for inference -- yielding practical statistical inference in near-linear time in the number of datapoints $n$. We adopt this approach in our experiments, allowing our procedures to run in linear time and remain practical, unlike previous work \citep{zhou2022boulevard, mentch2016randomforestci}. See Appendix \ref{app:random-projections}.

\vspace{-1ex}
\section{Theoretical Guarantees}
\vspace{-1ex}
\label{sec:theory}
As promised, we now formally present guarantees for convergence and asymptotic normality for Algorithms \ref{alg:random-dropout} and \ref{alg:structured-dropout}. Consider the (Lipschitz) nonparametric regression model introduced below:
\begin{aspt}[Lipschitz Nonparametric Regression]
    Let \(y = f(\bx) + \epsilon,\) where \(\mathbf{x}\) has density \(\mu\) satisfying $0 < c_1 \le \mu(\mathbf{x}) \le c_2 < \infty$ on its support. Further assume \(f\) is \(\alpha\)-Lipschitz continuous on \(\mathrm{supp}(\mu)\), and noise \(\epsilon\) is sub-Gaussian with variance proxy \(\sigma^2\).
    \label{aspt:nonparametric-regression}
\end{aspt}

\vspace{-1ex}
\paragraph{Finite-sample convergence to KRR.} We now show that our algorithms achieve finite-sample convergence to kernel ridge regression as the number of boosting rounds increases. At first glance, this is surprising -- greedy split selection and sequential ensemble construction create a highly nonlinear and time-dependent mapping. Fortunately, the regularization procedure of \cite{zhou2022boulevard} provides a way forward. 
For the special case of $p=0$ within Algorithm \ref{alg:random-dropout}, they show convergence to a kernel ridge regression via the stochastic contraction framework of \cite{almudevar2022stochasticcontractionmappingtheorem}, utilizing Assumptions \ref{aspt:svi} (structure–value isolation) and \ref{aspt:non-adaptivity} (non-adaptivity) to ensure stagewise stability. 
A key novelty within our analysis lies in showing that this convergence extends to Algorithms \ref{alg:random-dropout} and \ref{alg:structured-dropout}, even with the added complexity introduced by dropout and parallelism respectively. Despite these added sources of stochasticity and dependence, our use of the same regularization mechanism ensures that both algorithms inherit the stability required for convergence:
\begin{thm}[Finite Sample Convergence to KRR]
\label{thm:krrconv}
For fixed $\bX, \by$, under Assumptions \ref{aspt:svi}, \ref{aspt:non-adaptivity}, and \ref{aspt:nonparametric-regression},
$$\widehat{\by}_b \stackrel{\text{a.s.}}{\to} \left(\lambda^{-1} \bI+q \mathbb{E}\left[\bS_n\right]\right)^{-1} \mathbb{E}\left[\bS_n\right] \by \text{ for Alg. \ref{alg:random-dropout}},
\;
\widehat{\by}_b \stackrel{\text{a.s.}}{\to} \left(I+(K-1)\mathbb{E}\left[\bS_n\right]\right)^{-1} K\mathbb{E}\left[\bS_n\right] \by \text{ for Alg. \ref{alg:structured-dropout}},$$
as $b \to \infty$, where $q=1-p$ is one minus the dropout probability, $K$ is the number of trees grown in parallel, and $\bS_n$ is the kernel matrix induced by the tree structures within the ensemble.
\end{thm}
The proof, deferred to Appendix \ref{app:finite-sample-convergence}, departs from \cite{zhou2022boulevard} in two crucial ways. Within Algorithm \ref{alg:random-dropout}, dropout injects unbounded variance into the ensemble updates. To control this, we introduce the hard truncation function $\Gamma_M(y) = \operatorname{sign}(y)\min\{M,|y|\}$ into each residual update,
$
  \widehat{\by}_b 
    = \frac{b-1}{b}\,\widehat{\by}_{b-1}
    + \frac{\lambda}{b}\,\bS_b\bigl(\by - \Gamma_M(\widetilde{\by}_b)\bigr),
$
and then show that the probability of the partial ensemble escaping this cap vanishes as $b\to\infty$. Within Algorithm \ref{alg:structured-dropout}, truncating $\widetilde{\by}_{b,k}=\widehat{\by}_{b-1,K}-\Gamma_M(\frac{1}{b-1}\sum_{g=1}^{b-1}\widehat{\t}_{b,k})$ alone is insufficient. We further introduce a delay mechanism, requiring that $\t_{b,k}$ does not rely on $\t_{b,k-1}$, allowing us to apply Theorem  \ref{thm:stochastic-contraction-mapping} and parallelize tree training. 

\vspace{-1ex}
\paragraph{A central limit theorem.} Having established almost-sure convergence to fixed points, we now examine the asymptotic distribution of our estimators to justify the statistical procedures introduced in Section \ref{sec:statistics}. We seek a central limit theorem for the learner $\widehat{\f}_n^{D,P}$, trained on random $\bX_n, \by_n$, demonstrating that the predictions $\widehat{\f}_n^{D,P}(\bx)$ 
are asymptotically normal with mean $\f(\bx)$ as $n \to \infty$. To do so, we inherit the following assumptions from \cite{zhou2022boulevard}: two on the leaves to control the norm of the KRR weight vector $\r_n^{D,P}$, and a restriction on the tree distribution space:
\begin{aspt}[Bounded Leaf Diameter]
    \label{aspt:leaf-diameter}
    Write $\mathsf{diam}(A)=\sup_{x,y\in A}\norm{x-y}$. For any leaf $A$ in a tree with structure $q\in Q_n$, we need $\sup_{A \in q} \mathsf{diam}(A) =  O(d_n)$, where $d_n = O(n^{-{1}/{(d+1)}})$.
\end{aspt}
\vspace{-1ex}
\begin{aspt}[Increased Minimal Leaf Size]
    \label{aspt:ub-minimal-leaf-size}
    For any $\nu >0$, \(v_n =n^{-\frac{d+1}{d+2}+\nu}<n^{-\frac{d}{d+1}}=O(d_n^d)\).
\end{aspt}
\begin{aspt}[Restricted Tree Support]
    \label{aspt:restricted-tree-space}
    The cardinality of the tree space $Q_n$ is bounded by $O(n^{-1} \exp(0.5 n^{{1}/{(d+2)}-\nu} - n^{\alpha}))$, for some small $\alpha > 0$.
\end{aspt}
\vspace{-1ex}
Intuitively, Assumption \ref{aspt:ub-minimal-leaf-size} prevents leaves from becoming too small too quickly, which in turn controls the maximal coordinate of the weight vectors across random samples. Assumption \ref{aspt:restricted-tree-space}, guarantees that the complexity of possible tree partitions does not explode with \(n\). We also make assumptions on the regularity of tree splits in line with the notion of $\alpha$-regularity in \cite{athey2018generalizedrandomforests}:
\vspace{-1ex}
\begin{aspt}[Median Splitting Rules]
    \label{aspt:median-trees}
    The trees in Algorithm \ref{alg:structured-dropout} are split at the medians.
\end{aspt}
\vspace{-1ex}
We defer the details to Definition \ref{def:alpha_regular}. Under this splitting rule, the point-to-point collision probability is well-controlled, i.e. to say none of the leaves are big enough to contain most of the points. 
These conditions allow for uniform control over both leaf sizes and tree complexity, yielding the result:
\begin{thm}[Central Limit Theorem for Predictions]
	\label{thm:main}
	Let $\bx \in [0,1]^d$, $q \in (0,1]$, $K > 1$. As $n \to\infty$,
	$$
	\norm{\r^D_n}_2^{-1}\left(\widehat{\f}^D_n(\bx) - \lambda^{-1}(1+\lambda q)\f(\bx)\right) \stackrel{d}{\longrightarrow} \mathcal{N}(0,\sigma^2), \quad
	\norm{\r^P_n}_2^{-1}\left(\widehat{f}^P_n(\bx) - f(\bx)\right) \stackrel{d}{\longrightarrow} \mathcal{N}(0,\sigma^2).
	$$
\end{thm}
\vspace{-1ex}
We prove this in two stages. In the first, with proof deferred to Appendix \ref{app:conditional-clt}, we show asymptotic normality of the predictions around the noiseless KRR predictions $\langle \r_n^{D,P}(\bx), \f(\bX_n)\rangle$, conditional on the training data $\bX_n, \by_n$ and test point $\bx$, by applying the Lindeberg-Feller CLT to \(\langle\mathbf{r}_n^{D,P}(\bx), \varepsilon\rangle\). 
Next, we show that $\langle \r_n^{D,P}(\bx), \f(\bX_n)\rangle$ converges to the underlying ground truth function at a sufficiently fast rate in Appendices \ref{app:random-clt} and \ref{app:main-theorem}. For Algorithm \ref{alg:random-dropout}, although $\r_n^D$ now involves $[\frac1\lambda \bI+q\bK_n]^{-1}$, this still preserves the rate $\|\r_n\|_2\asymp n^{-1/(2(d+1))}$, allowing us to show $\langle \br_n,f(\bX_n)\rangle-\frac{\lambda q}{1+\lambda q} \f(\bx)=o_p(\|\r_n\|_2)$. The Lindeberg–Feller verification then proceeds identically to the non‑dropout case. However, to give a CLT for Algorithm \ref{alg:structured-dropout}, controlling the influence from distant training points is challenging, though we show that Assumption \ref{aspt:median-trees} is sufficient to do so. 

\vspace{-1ex}
\paragraph{Understanding the result.} 
The rate of convergence of the CLT in Theorem \ref{thm:main} depends on the norm of $\r_n^{D,P}$. This is controlled as follows, yielding generalization bounds for Algorithms \ref{alg:random-dropout} and \ref{alg:structured-dropout}:
\begin{lem}[Rate of Convergence]
Let $\mathbf{B}_n := \{i  : \norm{\bx-\bx_i}\leq d_n\}$ be the points within distance $d_n$ from test point $\bx$. If $\left| \mathbf{B}_n \right|  = \Ome{ n \cdot d_n^d }$, then
	$
	\norm{\bk_n}_2 = \Theta(n^{-\frac{1}{2}\frac{1}{d+1}}),
    \norm{\r^{D,P}_n}_2 = \Theta(n^{-\frac{1}{2}\frac{1}{d+1}}).
    $
\label{lem:rateofr}
\end{lem}
\begin{cor}
\label{cor:asymptotic-risk}
$
\E[(\frac{1+\lambda q}{\lambda}\widehat{\f}^D_n(\bx)-\f(\bx))^2]
\lesssim (\frac{1+\lambda q}{\lambda})^2\sigma^2 n^{-\frac{1}{d+1}}$ and $ 
\E[(\widehat f^P_n(\bx)-f(\bx))^2]
\lesssim \sigma^2 n^{-\frac{1}{d+1}}.
$
\end{cor}
\begin{cor}
\label{cor:non-asymptotic-risk}
If $n \geq \Omega\left(\frac{\log(1/\delta)}{\epsilon^{d+1}}\right)$, $b \geq \Omega\left(\frac{n\lambda^2M^2}{\epsilon^{2}\delta}\right)$, then w.p. at least $1-\delta$, $|\widehat{\f}_n^{(b)}(\bx) - \f(\bx)| \leq \epsilon$.
\end{cor}
The first is an asymptotic risk bound, while the second is a nonasymptotic PAC guarantee. Recall that $M$ is the truncation level.
The results above recover the minimax rate for nonparametric regression on $1/2$-Holder smooth functions \citep{stone1982} -- quadratically worse than the rate for Lipschitz functions (the setting we are in). We believe this can be improved -- trees should inherit adaptivity to the intrinsic dimension as an adaptive nearest neighbor method, but we leave this for future work.

Furthermore, Theorem \ref{thm:main} yields asymptotic coverage of our intervals constructed in Section \ref{sec:statistics}:
\begin{cor}
    \label{cor:cov-rate}
    The CIs, PIs and RIs achieve $1-\alpha$ almost sure pointwise coverage as $n,b \to \infty$.
\end{cor}

We show in Lemma \ref{lem:rateofr} that the rate at which $\lVert\r_n^{D,P}\rVert_2$ grows is the same for Algorithms \ref{alg:random-dropout} and \ref{alg:structured-dropout}. This shows that our methods achieve an asymptotic relative efficiency relative to \cite{zhou2022boulevard} by up to a factor of $4$ for Algorithm \ref{alg:random-dropout}, and at least a factor of $4$ for Algorithm \ref{alg:structured-dropout}:
\begin{cor}[Asymptotic Relative Efficiency]
    \label{cor:are}
    Write $\widehat{\f}^B$ for the scaled predictions made by vanilla Boulevard \citep{zhou2022boulevard} with the same learning rate $\lambda \in (0,1]$ as Algorithm \ref{alg:random-dropout}. Then,
    \begin{tightcenter}
        $
    {\text{Var}(\frac{1+\lambda}{\lambda}\widehat{\f}^B)}/{\text{Var}(\frac{1+\lambda q}{\lambda}\widehat{\f}^D)}=(\frac{1+\lambda}{1+\lambda q})^2 \in [1,4],
    \; 
    {\text{Var}(\frac{1+\lambda}{\lambda}\widehat{\f}^B)}/{\text{Var}(\widehat{\f}^P)}=\left(\frac{1+\lambda}{\lambda}\right)^2 \in [4,\infty).
    $
    \end{tightcenter}
\end{cor}
\vspace{-2ex}
This has the implication that Algorithms \ref{alg:random-dropout} and \ref{alg:structured-dropout} achieve increased signal recovery in the square root of the ARE relative to vanilla Boulevard. The intuition for this improvement is as follows. For any choice of $\lambda$ within vanilla Boulevard, there exists a choice of $p$ within Algorithm \ref{alg:random-dropout} that requires a smaller rescaling. On the other hand, as Algorithm \ref{alg:structured-dropout} requires no rescaling, it can achieve an unbounded improvement in relative efficiency over vanilla Boulevard as we take $\lambda \to 0$.

\vspace{-2ex}
\section{Numerical Experiments}
\label{sec:experiments}
\vspace{-2ex}

We have presented a rich theory in the previous section. Alongside convergence to kernel ridge regression and a central limit theorem for predictions, we provide asymptotic risk bounds and nonasymptotic PAC guarantees, and show asymptotic coverage for our intervals. 
In light of this, we present a series of numerical experiments to justify our algorithms and methods empirically. 

\begin{figure}[t]
    \makebox[\textwidth][c]{\includegraphics[width=1.1\linewidth]{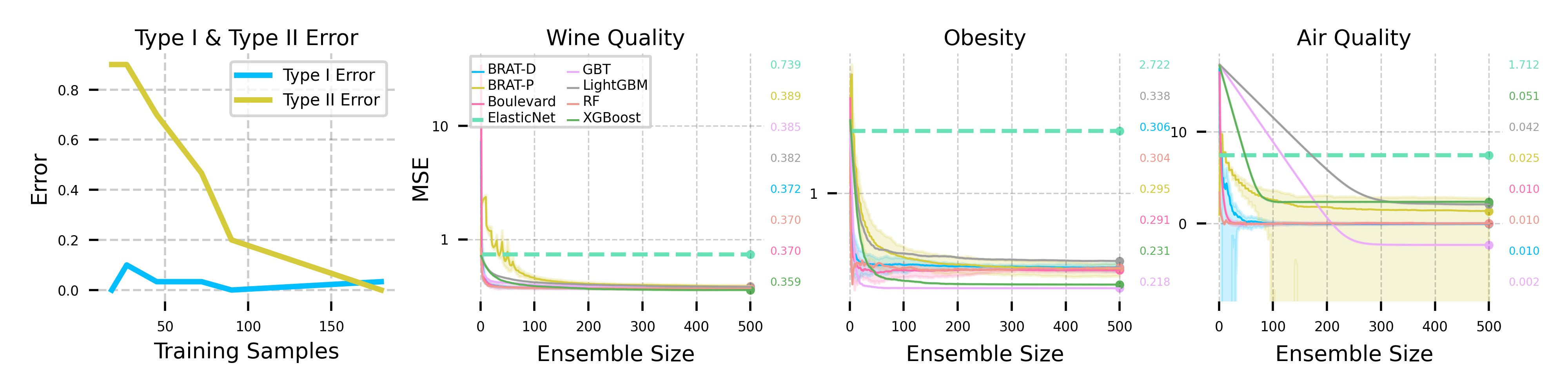}}
    \vspace{-5ex}
    \caption[Size & power / MSE races]{%
      \textbf{Left:} Type I and Type II error of the hypothesis test for variable importance against training set size. Test set is of the same size. Error rates computed over 30 trials.
      \textbf{Right:} Comparison of MSE achieved by various machine learning algorithms on different datasets, with hyperparameters tuned for each by Optuna. Shaded area depicts two standard deviations over 5 trials.
    }
    \label{fig:mse-tests}
    \vspace{1ex}
\end{figure}


\vspace{-2ex}
\paragraph{Predictive accuracy.} The first, and most natural, thing to do is to examine the performance of our algorithms against a handful of competitors in terms of test MSE on nine datasets from the UCI Machine Learning Repository in Figure \ref{fig:mse-tests}. Our competitors include popular methods such as XGBoost \citep{chen2016xgboost}, LightGBM \citep{ke2017lightgbm}, as well as random forests and vanilla boosting. We also compare to vanilla Boulevard \citep{zhou2022boulevard}, and an elastic net regression as a baseline. All hyperparameters were tuned with Optuna \citep{optuna_2019}, reported in Appendix \ref{app:mse}. There is no clear winner in general, but XGBoost is a consistently strong performer, and random forests are clearly more suited to some datasets than others. Neither Algorithm \ref{alg:random-dropout} nor Algorithm \ref{alg:structured-dropout} consistently outperform each other, but Algorithm \ref{alg:structured-dropout} occasionally exhibits instability on some datasets. Regardless, both of our algorithms are competitive in terms of final MSE and rate of convergence, and exhibit the ability to be tuned to be closer to boosting (Wine Quality) or random forests (Air Quality), providing increased flexibility by interpolating between the two.

Our results in Figure \ref{fig:mse-tests} are somewhat unfair to Algorithm 2, as the x-axis is in the ensemble size and not the number of boosting rounds. The number of boosting rounds that Algorithm 2 encounters can be as few as 31 on Abalone, in contrast to the 500 boosting rounds all other algorithms enjoy. Given the same number of boosting rounds, we would certainly expect an improvement.

\begin{figure}[t]
    \makebox[\textwidth][c]{\includegraphics[width=1.1\linewidth]{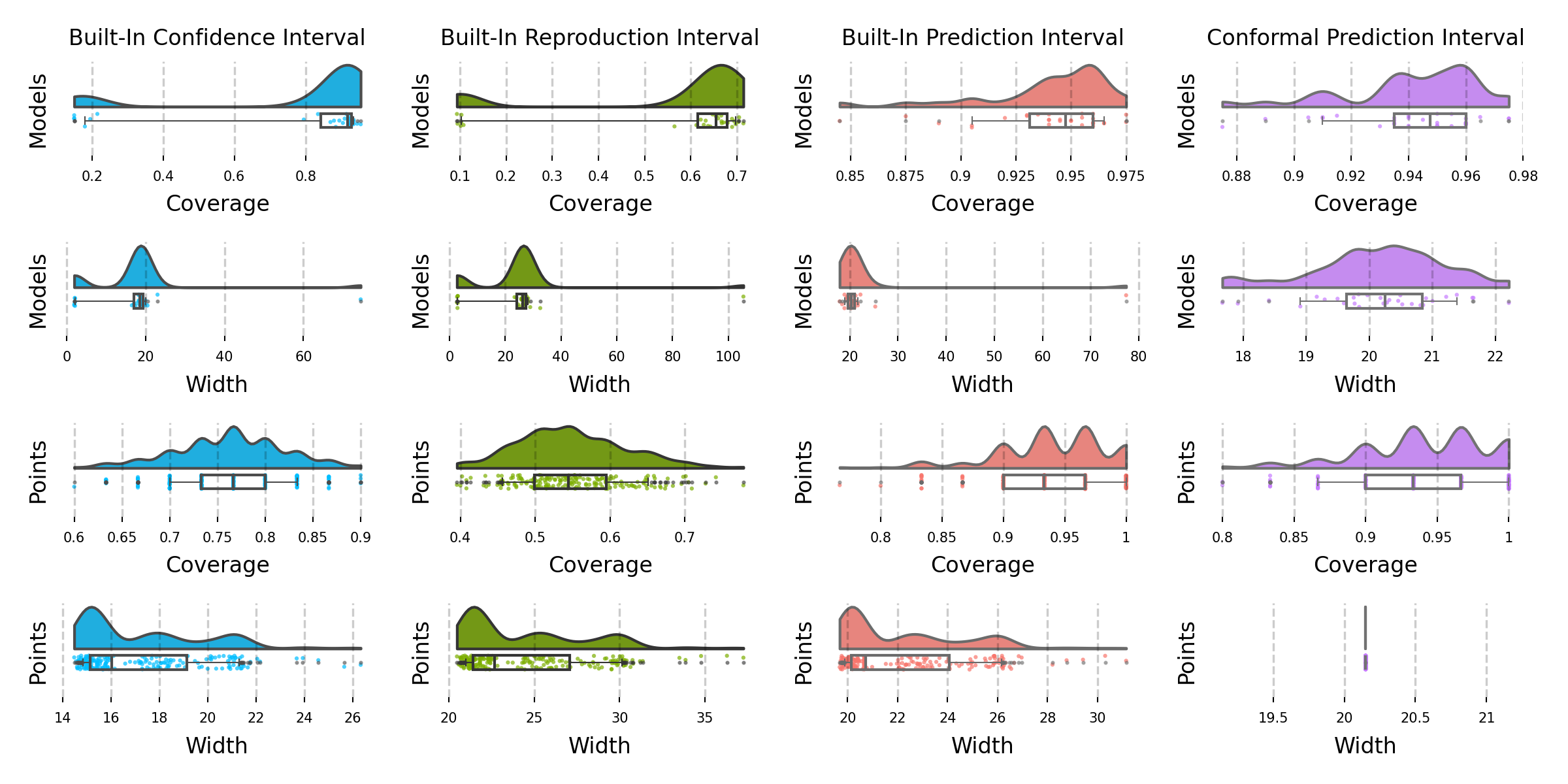}}
    \vspace{-4ex}
    \caption{Points in the first two rows represent one of $30$ models, depicting the fraction of test points that fell inside the interval generated by the model (marginal coverage) and average width for each model. Points in the last two represent one of $100$ test points, depicting the fraction of models with intervals containing the test point (conditional coverage) and average width for each point. Results from Algorithm \ref{alg:random-dropout} with $200$ trees, learning rate $0.6$, subsampling $0.8$, dropout $0.3$, max depth $4$.}
    \label{fig:rainclouds}
\end{figure}
    
\vspace{-2ex}    
\paragraph{Variable importance tests.} We examine their performance by conducting a simulation study, testing the null $H_0:w=0$ on data from $\f(\bx)=4 x_1 -x^2_2+w bx_3$. We fit Algorithm \ref{alg:random-dropout} on data generated from $\f(\bx)$ and $\g(\bx)=4 x_1 - x^2_2$, with $100$ trees, $\lambda=1$, subsampling rate $1$, dropout rate $0.95$, and a max depth of $6$. The size and power of our test presented in Section \ref{sec:statistics} are depicted in the left panel within Figure \ref{fig:mse-tests}. Our test performs very well, maintaining appropriate size control throughout while the Type II error decreases quickly. Empirically, increasing the dropout rate increases power.
\vspace{-2ex}  
\paragraph{Coverage of intervals.}
\label{sec:coverage-experiments}
Other than the Nystr\"{o}m approximations for fast and practical interval computation in Appendix \ref{app:random-projections}, there is one more tweak that we encourage users to make in practice. We demonstrate in Corollary \ref{cor:cov-rate} that our intervals have asymptotic coverage, but this says nothing about its finite sample properties. Fortunately, \cite{candes2020adaptivecoverage} yields a simple enhancement. Since we estimate $\widehat{\sigma}$ through computing the residuals on a hold-out calibration set, we reuse the calibration set to adaptively grow or shrink our intervals according to the empirical prediction interval coverage on the calibration set. This procedure converges quickly (amounting to a doubling trick and binary search), and empirically is very helpful in increasing robustness and finite-sample performance.

To examine the empirical performance of our intervals, we consider the Friedman function $\f(\bx) = 10\sin(\pi \bx_{1}\bx_{2})+20(\bx_{3}-0.5)^{2}+5\bx_{5}-10$ \citep{friedman2000greedy}. This is depicted as a raincloud plot in Figure \ref{fig:rainclouds}. Importantly, we consider two notions of coverage: marginal coverage, where coverage is averaged over test points, and conditional coverage, where the coverage is conditional on a test point. The first two rows depict the former, while the last two depict the latter. Our intervals perform reasonably well, attaining nominal coverage with the exception of the reproduction interval's coverage and the confidence interval's conditional coverage. In particular, the prediction interval performs very well after the adaptive coverage adjustment. Importantly, unlike the conformal benchmark that has a constant width at each point, our prediction intervals have different interval widths at each point, allowing users to identify `hard'' examples. See Appendix \ref{app:rain-clouds} for other hyperparameter choices.
\vspace{-4ex}
\paragraph{CI tuning.} 
Reducing subsample rate \(\xi\) inflates variance proportionally, while deeper trees yield narrower intervals as finer splits yield sparser \(\k_n(\bx)\) and smaller \(\|\r_n\|\).  
Variance grows with shrinkage \(1/\lambda\) and falls as dropout \(p\) increases. Choosing moderate \(\lambda\) and increasing \(p\) yields stabler CIs. Setting $p \to 0$ empirically leads to wider intervals and less power in the hypothesis test for variable importance. This can be seen in Figure \ref{fig:dropout-tests}.


\vspace{-2ex}
\begin{figure}[H]
    \centering
    \includegraphics[width=0.49\linewidth]{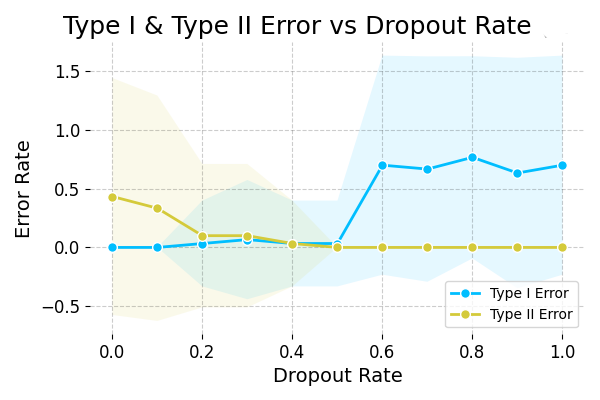}
    \includegraphics[width=0.49\linewidth]{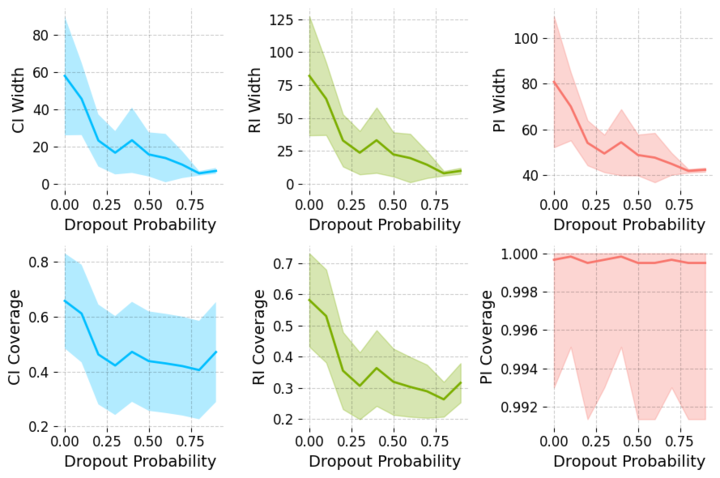}
    \vspace{-1ex}
    \caption{Comparison of size and power of statistical inference procedures against dropout probability. Conformal adjustment is not performed (for clearer benchmarking, hence the under/overcovarage), and we plot the coverage and width across all datapoints and repetitions.}
    \label{fig:dropout-tests}
\end{figure}
\vspace{-2ex}

\vspace{-4ex}
\section{Conclusions and Further Work}
\vspace{-2ex}
\label{sec:discussions}
We incorporate dropout and parallel boosting into the regularization procedure of \cite{zhou2022boulevard}, allowing for increased signal recovery and improved asymptotic variance. With the CLTs our algorithms enjoy, we construct confidence and prediction intervals, and hypothesis tests for variable importance. Numerical experiments empirically validate our algorithms and statistical procedures.

Our theoretical guarantees depend on structure–value isolation, non‑adaptivity, and tree regularity assumptions. Relaxing these -- e.g., via honesty‐based splitting or quantifying mild adaptivity -- would broaden applicability. Extending inference beyond regression (to classification, survival, or structured outputs) also poses new challenges: non‑quadratic losses, different tree behavior, and the need for new CLTs. Doing so under weaker assumptions is a welcome direction for future work.

\bibliographystyle{apalike}
\bibliography{ref}

\appendix

\section{Matrix Sketching}
\label{app:random-projections}

We first sketch the kernel inversion where we approximate $\widehat{\bK}_n^D \approx \widetilde{\bK}_n^D = \widehat{\bK}_n^D\bS(\bS^\top \widehat{\bK}_n^D\bS)^{\dagger} \bS^\top \widehat{\bK}_n^D$:
\begin{align*}
\widehat{\r}_n^D(\bx)^\top \by_n &= \widehat{\k}_n^D(\bx)^\top (\lambda^{-1}\bI + q\widehat{\bK}_n^D)^{-1} \by_n  \\
&\approx \widehat{\k}_n^D(\bx)^\top (\lambda^{-1}\bI + q \widetilde{\bK}_n^D)^{-1} \by_n  \\
&= \widehat{\k}_n^D(\bx)^\top (\lambda^{-1}\bI + q \widehat{\bK}_n^D\bS(\bS^\top \widehat{\bK}_n^D\bS)^{\dagger} \bS^\top \widehat{\bK}_n^D)^{-1} \by_n  \\
&= \widehat{\k}_n^D(\bx)^\top \left(\lambda \bI - \lambda^2 \widehat{\bK}_n^D\bS (q^{-1}\bS^\top \widehat{\bK}_n^D\bS + \lambda (\widehat{\bK}_n^D\bS)^\top \widehat{\bK}_n^D\bS)^{-1}\bS^\top \widehat{\bK}_n^D\right) \by_n \\
&= \widehat{\k}_n^D(\bx)^\top \widehat{\bLambda}^D_n \by_n = \widehat{\k}_n^D(\bx)^\top \widehat{\balpha}_n,
\end{align*}
where we define the $n$-dimensional coefficient vector
$$\widehat{\balpha}_n = \left(\lambda \bI - \lambda^2 \widehat{\bK}_n^D\bS (q^{-1}\bS^\top \widehat{\bK}_n^D\bS + \lambda (\widehat{\bK}_n^D\bS)^\top \widehat{\bK}_n^D\bS)^{-1}\bS^\top \widehat{\bK}_n^D\right) \by_n,$$
and write 
$$\widehat{\bLambda}_n^D = \left(\lambda \bI - \lambda^2 \widehat{\bK}_n^D\bS (q^{-1}\bS^\top \widehat{\bK}_n^D\bS + \lambda (\widehat{\bK}_n^D\bS)^\top \widehat{\bK}_n^D\bS)^{-1}\bS^\top \widehat{\bK}_n^D\right) \in \R^{n \times n}$$
for our sketched estimate of the inverse kernel matrix.
Computing this takes $O(ns^2)$ time when $\bS$ is a subsampling matrix \citep{musco2017recursivenystrom}.

Now we further sketch the coefficients $\widetilde{\balpha}_n = (\bS^\top \widehat{\bK}_n^D\bS)^{\dagger} \bS^\top \widehat{\bK}_n^D \widehat{\balpha}_n \in \R^s$, in line with Appendix C of \cite{musco2017recursivenystrom}:
\begin{align*}
    \widehat{\k}_n^D(\bx)^\top \widehat{\balpha}_n
    &\approx \left\langle \bS^\top \widehat{\k}_n^D(\bx) , (\bS^\top \widehat{\bK}_n^D\bS)^{\dagger} \bS^\top \widehat{\bK}_n^D \widehat{\balpha}_n \right\rangle = \left\langle \widetilde{\k}_n^D(\bx) ,  \widetilde{\balpha}_n \right\rangle .
\end{align*}
When $\widetilde{\balpha}_n = (\bS^\top \widehat{\bK}_n^D\bS)^{\dagger} \bS^\top \widehat{\bK}_n^D \widehat{\balpha}_n$ is precomputed, we can make a new KRR point prediction in only $O(s)$ time. This is because we only need to compute the $s$ coordinates of the vector $\widetilde{\k}_n^D(\bx)  = \bS^\top \widehat{\k}_n^D(\bx) \in \R^s$, taking $s$ kernel evaluations, and multiplication with $\widetilde{\balpha}_n = (\bS^\top \widehat{\bK}_n^D\bS)^{\dagger} \bS^\top \widehat{\bK}_n^D \widehat{\balpha}_n \in \R^s$ can be done in $O(s)$ time.

\paragraph{$O(s^2)$ time inference.} So in order to perform inference in sublinear time, we first precompute in $O(ns^2)$ time:
$$\widetilde{\bLambda}_n^D = \widehat{\bLambda}_n^D \widehat{\bK}_n^D \bS (\bS^\top \widehat{\bK}_n^D\bS)^{\dagger} \in \R^{n \times s}, \qquad \widehat{\bSigma}_n^D = \left(\widetilde{\bLambda}_n^D\right)^\top \widetilde{\bLambda}_n^D \in \R^{s \times s}.$$
To obtain a sketched estimate of $\lVert{\r}_n^D(\bx)\rVert_2$ on a new test point $\bx$, we compute the $s$ coordinates of the vector $\widetilde{\k}_n^D(\bx) = \bS^\top \widehat{\k}_n^D(\bx) \in \R^s$, and obtain an estimate 
$$\lVert\widetilde{\r}_n^D(\bx)\rVert_2 = \sqrt{\widetilde{\k}_n^D(\bx)^\top \widehat{\bSigma}_n^D\widetilde{\k}_n^D(\bx)}$$ in $O(s^2)$ time. The equivalent expression $\lVert\widetilde{\r}_n^D(\bx)\rVert_2 = \lVert \widetilde{\bLambda}^D_n \widetilde{\k}_n^D(\bx) \rVert_2$ takes $O(n)$ time.

\paragraph{Sketched hypothesis testing.} Given a hold-out dataset $(\bX_m, \by_m)$, we test the null hypothesis:
$$H_0 : \f(\bx_j) = \g({\bx}_j) \text{ for all } j=1,...,m \text{ against } H_1 : \f(\bx_j) \neq \g_n({\bx}_j) \text{ for some } j$$
by comparing $\widehat{\f}_{n,1}(\bx_j), \widehat{\f}_{n,2}({\bx}_j)$. However, given $m$ test points on a test dataset $(\bX_m, \by_m)$, we can subsample $r$ points with a new subsampling matrix $\r$ to compute the doubly subsampled kernel matrix $\widetilde{\bkappa}_{n,1}^D(\bX_m) = \bS^\top \widehat{\k}_n^D(\bX_m)\r \in \R^{s\times r}$, which we only need to perform $s \times r$ kernel operations to compute. So the difference statistic is then given by
$$\widetilde{\d}_m = \widetilde{\bkappa}_{n,1}(\bX_m)^\top \widetilde{\balpha}_{n,1} - \widetilde{\bkappa}_{n,2}({\bX}_m)^\top \widetilde{\balpha}_{n,2} = \left(\widetilde{\bLambda}_{n,1}^D \widetilde{\bkappa}_{n,1}(\bX_m) - \widetilde{\bLambda}_{n,2}^D \widetilde{\bkappa}_{n,2}(\bX_m) \right)^\top\by \in \R^{r}.$$

Modulo subsampling and Nystr\"{o}m approximation error, this is multivariate normal under the null: $\widetilde{\d}_m^D  \sim \gN_{r}\left(\mathbf{0}, \sigma^2 \widetilde{\bXi}_n^D \right)$, with covariance matrix 
$$\widetilde{\bXi}_n^D = \left(\widetilde{\bLambda}_{n,1}^D \widetilde{\bkappa}_{n,1}(\bX_m)^\top - \widetilde{\bLambda}_{n,2}^D \widetilde{\bkappa}_{n,2}(\bX_m)^\top\right)^\top\left(\widetilde{\bLambda}_{n,1}^D \widetilde{\bkappa}_{n,1}(\bX_m)^\top - \widetilde{\bLambda}_{n,2}^D \widetilde{\bkappa}_{n,2}(\bX_m)^\top\right) \in \R^{r \times r}.$$

We then have the test statistic:
$$\widehat{\sigma}^{-2} \widetilde{\d}_m^{D \, \top} (\widetilde{\bXi}_n^{D})^{-1} \widetilde{\d}_m^{D} \sim \chi^2_{r}.$$

The runtime of the test is as follows. It takes $O(ns^2)$ precomputation time to form $\widetilde{\bLambda}_{n,1}^D, \widetilde{\bLambda}_{n,2}^D \in \R^{n\times s}$, $O(nsr)$ time for the multiplication with $\widetilde{\bkappa}_{n,1}(\bX_m), \widetilde{\bkappa}_{n,2}(\bX_m)$, and $O(n)$ time for multiplication with $\by$. Once one has the noise estimate, inverting the covariance matrix and computing the test points takes $O(r^3)$ time. 

\section{Finite Sample Convergence}
\label{app:finite-sample-convergence}

\subsection{Stochastic Contraction Mapping Theorem}
To show convergence for both algorithms, we will show both of them are a stochastically contracted process. The criteria of deciding such process can be found in \cite{almudevar2022stochasticcontractionmappingtheorem}, and we introduce the theorem below.

\begin{thm}
    \label{thm:stochastic-contraction-mapping}
    Given $\mathbb{R}^d$-valued stochastic process $\{Z_t\}_{t \in \mathbb{N}}$, a sequence of $0 < \lambda_t \leq 1$, define
    \begin{gather*}
    \mathcal{F}_0 = \emptyset, \mathcal{F}_t = \sigma(\bz_1,\dots, \bz_t), \\
    \mathbf{\epsilon}_t =  \bz_t - \E[\bz_t | \mathcal{F}_{t-1}].
    \end{gather*}
    We call $\bz_t$ a stochastic contraction if the following properties hold
    \begin{enumerate}
    \item Vanishing coefficients $$\sum_{t=1}^{\infty} (1-\lambda_t) = \infty, \mbox{ i.e. } \prod_{t=1}^{\infty}\lambda_t = 0.$$
    \item Mean contraction $$||\E[\bz_t|\mathcal{F}_{t-1}]|| \leq \lambda_t \norm{\bz_{t-1}}, a.s..$$
    \item Bounded deviation $$\sup \norm{\epsilon_t} \to 0, \quad \sum_{t=1}^{\infty}\E[\norm{\epsilon_t}^2] < \infty.$$
    \end{enumerate}
    In particular, a multidimensional stochastic contraction exhibits the following behavior
    \begin{enumerate}
    \item Contraction $$\bz_t \xlongrightarrow{a.s.} 0.$$
    \item Kolmogorov inequality
    \begin{align}
    \label{fml:kolmax}
    P\left( \sup_{t \geq T}\norm{\bz_t} \leq \norm{\bz_T} + \delta  \right) \geq 1- \frac{4\sqrt{d}\sum_{t=T+1}^{\infty}\E[\epsilon_t^2]}{\min\{\delta^2, \beta^2\}}
    \end{align}
    \blued{holds for all $T, \delta > 0$ s.t. }$\beta = \norm{\bz_T} + \delta - \sqrt{d} \sup_{t > T} \norm{\epsilon_t} > 0$.
    \end{enumerate}
\end{thm}

In the following proof we benefit from the a.s. convergence of the difference vector to 0. And the Komolgorov inequality gives us a PAC argument in \ref{sec:theory}.

\subsection{Subsampling}

We will also introduce a lemma to regularize the expected tree kernel in a finite sample case with respect to subsampling rate.

\begin{lem}
	\label{lem:kernelrate}
	Considering a subsampled regression tree. Assume that each leaf contains no fewer than $n^{\frac{1}{d+2}}$ sample points before subsampling. If we are subsampling at rate at least $\xi=n^{-\frac{1}{d+2}\log n}$, then the expected structure vector's norm follows the rate below:
	$$
	\norm{\k_n}=1-O\left(\frac{1}{n}\right)
	$$
\end{lem}

\begin{proof}
	By \cite{zhou2022boulevard}, we know that $\norm{\mathbb{E}_{q,w}[\bS_n]}\leq 1$. By this we know that at least $\norm{\k_n}_1\leq1$. The task remained is to give the distance between $\norm{\k_n}_1$ and 1.
	
	Define the subsampling index set as $w$. To see clearly why a rate containing $n$ would exists, recall the definition of a structure vector with subsampling,for any $x \in A_j$,
	\[ 
	\s_{n,k}(\bx)  = \s_{n,k}(x; w) =\frac{\mathbbm{1}(\bx_k \in A_j)\mathbbm{1}(k \in w)}{\sum_{i=1}^n \mathbbm{1}(\bx_i \in A_j)\mathbbm{1}(i \in w)} = \frac{\mathbbm{1}(\bx_k \in A_j)\mathbbm{1}(k \in w)}{\sum_{\bx_i \in A_j} \mathbbm{1}(i \in w)}.
	\]
	The structure vector is given by
	\[
	\s_n(x)=[\s_{n,1}(x),\s_{n,2}(x),\cdots,\s_{n,n}(x)]^{T}
	\]
	Taking expectation we have the kernel vector 
	\[
	\k_n^T= \frac{\mathbb{P}(\bx_k \in A_j)\mathbb{P}(k \in w)}{\sum_{\bx_i \in A_j} \mathbb{P}(i \in w)}
	\] 
	Notice the expectation is taken with respect to both structure distribution and the subsampling distribution. And this vector is not defined if one leaf is empty, since the denominator would be 0. Let the indicator $I_A$ denote whether the leaf of interest is empty. It values on 0 if the leaf is empty and it values on 1 if the value is not empty. The $L_1$ norm is given by:
	\begin{align*}
		\norm{\k_n}_1&=\sum_{k=1}^{n} \mathbb{E}_{q,w}[\s_{n,k}(x)]\\
		&=\sum_{k=1}^{n} (\mathbb{E}_{q,w}[\s_{n,k}(x)|\mathbbm{1}_A=0]\mathbb{P}(\mathbbm{1}_A=0)+\mathbb{E}_{q,w}[\s_{n,k}(x)|\mathbbm{1}_A=1]\mathbb{P}(\mathbbm{1}_A=1))\\
	\end{align*}
	If the leaf is empty, then we will have to define $\mathbb{E}_{q,w}[\s_{n,k}(x)|\mathbbm{1}_A=0]=0$, otherwise the denominator would be 0 and the whole structure vector is undefined.  If the leaf is not empty, we would then want to define the conditional expectation such that it sums up to 1:
	\[
	\sum_{k=1}^{n}\mathbb{E}[s_{n,k}(x)|I_A=1]=1
	\]
	Hence the $L_1$ norm above simplifies to:
	\[
	\norm{\k_n}_1=\sum_{k=1}^{n}\mathbb{E}[\s_{n,k}(x)|\mathbbm{1}_A=1]\mathbb{P}(\mathbbm{1}_A=1)=1-\mathbb{P}(\mathbbm{1}_A=0)
	\]
	Hence the problem is reduced to finding the probability of missing all subsampled points in a leaf of interest.
	
	By assertion, the subsample size should be $\xi n=n^{\frac{d+1}{d+2}}\log n$. Consider the probability that one leaf misses all sample points $p(n,\xi)$. That is, all the subsampled points  sampled are points that are in this particular leaf. Thus we are only choosing with choice $\binom{n-n^{\frac{1}{d+2}}}{\theta n}$.
	\begin{align*}
		\mathbb{P}(n,\xi)&=\frac{\binom{n-n^{\frac{1}{d+2}}}{\xi n}}{\binom{n}{\xi n}}\\
		&=\frac{(n-\xi n)(n- \xi n -1)(n - \xi n -2)\cdots(n-\xi n -n^{\frac{1}{d+2}}+1)}{n(n-1)(n-2)\cdots(n-n^{\frac{1}{d+2}}+1)}\\
		&\leq \left(\frac{n-\xi n}{n-n^{\frac{1}{d+2}}}\right)^{n^{\frac{1}{d+2}}}\\
		&=\left(\frac{1-\xi}{1-n^{-\frac{d+1}{d+2}}}\right)^{n^{\frac{1}{d+2}}}\\
		&=\left(\frac{1-n^{-\frac{1}{d+2}}\log n}{1-n^{-\frac{d+1}{d+2}}}\right)^{n^{\frac{1}{d+2}}}\\
		&=\left(\frac{1}{1-n^{-\frac{d+1}{d+2}}}\right)^{n^{\frac{1}{d+2}}}\cdot \left(1-n^{-\frac{1}{d+2}}\log n\right)^{n^{\frac{1}{d+2}}}\\
	\end{align*}
	Notice that
	\begin{align*}
		\left(\frac{1}{1-n^{-\frac{d+1}{d+2}}}\right)^{n^{\frac{1}{d+2}}}&\leq \left(\frac{1+n^{-\frac{d+1}{d+2}}}{1-n^{-\frac{d+1}{d+2}}+n^{-\frac{d+1}{d+2}}}\right)^{n^{\frac{1}{d+2}}}\\
		&=\left(1+n^{-\frac{d+1}{d+2}}\right)^{n^{\frac{1}{d+2}}}\\
		&\leq \left(1+n^{-\frac{1}{d+2}}\right)^{n^{\frac{1}{d+2}}}\\
		&\leq e
	\end{align*}
	And meanwhile we recognize
	\begin{align*}
		\left(1-n^{-\frac{1}{d+2}}\log n\right)^{n^{\frac{1}{d+2}}}&=\left[\left(1-n^{-\frac{1}{d+2}}\log n\right)^{n^{\frac{1}{d+2}}\frac{1}{\log n}}\right]^{\log n}\\
		&=\left(\frac{1}{e}\right)^{\log n}\\
		&=\frac{1}{e^{\log  n}}\\
		&=\frac{1}{n}
	\end{align*}
	Hence we have
	\begin{align*}
		&\left(\frac{1}{1-n^{-\frac{d+1}{d+2}}}\right)^{n^{\frac{1}{d+2}}}\cdot \left(1-n^{-\frac{1}{d+2}}\log n\right)^{n^{\frac{1}{d+2}}}\\
		&\leq e \cdot \left(1-n^{-\frac{1}{d+2}}\log n\right)^{n^{\frac{1}{d+2}}}\\&=O\left(\frac{1}{n}\right)
	\end{align*}
	Hence the probability of a single leaf missing all subsampled points is of rate $O\left(\frac{1}{n}\right)$. And hence $\norm{\k_n}_1=1-\O{\frac{1}{n}}$.
\end{proof}

\subsection{Proof for Theorem \ref{thm:krrconv}}

Combing all the results above, the proof for Theorem \ref{thm:krrconv} is as follow.

\begin{proof}
    To show Algorithm \ref{alg:random-dropout} is a stochastic contraction mapping, define $\mathbf{z}_t:=\widehat{\by}_{t+1}-\widehat{\by}^*$. And define Define $\widetilde{\by}_b:=\frac{1}{b}\sum_{k=1}^{b}X_k\widehat{t}_k, X_k \quad \overset{i.i.d.}{\sim}\text{Ber}(q)$. To check mean contraction, first notice as $b \to \infty$, $\text{tr(Var($\widetilde{\by}_b$))}\leq\frac{q(1-q)nM}{b}\to 0$. Hence with $b$ large enough, we can get rid off the truncation.
    \begin{align*}
    \norm{\E[\mathbf{z}_b | \mathcal{F}_{b-1}]} & = \norm{\E\left[ \frac{b-1}{b}  \widehat{\by}_{b-1} + \frac{\lambda}{b} \bS_n(\by - \Gamma_M(\widetilde{\by}_{b-1})) - \widehat{\by}^*\Big| \mathcal{F}_{b-1} \right]}\\
		& = \norm{\frac{b-1}{b} (\widehat{\by}_{b-1}-\widehat{\by}^*) + \frac{\lambda}{b} \E[\bS_n](\by - \mathbb{E}[\Gamma_M(\widetilde{\by}_{b-1})|\mathcal{F}_{b-1}]) - \frac{1}{b} \widehat{\by}^*}\\
		& \leq \frac{b-1}{b} \norm{\widehat{\by}_{b-1}-\widehat{\by}^*} + \norm{\frac{\lambda}{b} \E[\bS_n](\by - \mathbb{E}[\Gamma_M(\widetilde{\by}_{b-1})|\mathcal{F}]) - \frac{\lambda}{b}\E[\bS_n](\by - q\widehat{\by}^*)} \\
		&= \frac{b-1}{b}\norm{\mathbf{z}_{b-1}}+\norm{\frac{\lambda}{b}\mathbb{E}[\bS_n](q\widehat{\by}^*-\mathbb{E}[\Gamma_M(\widetilde{\by}_{b-1})|\mathcal{F}_{b-1}])}\\
            &= \frac{b-1}{b}\norm{\mathbf{z}_{b-1}}+\norm{\frac{\lambda}{b}\mathbb{E}[\bS_n](q\widehat{\by}^*-\mathbb{E}[\widetilde{\by}_{b-1}|\mathcal{F}_{b-1}])}\\
            & \leq \frac{b-1+\lambda q}{b} \norm{\widehat{\by}_{b-1}-\widehat{\by}^*} \leq \norm{\widehat{\by}_{b-1}-\widehat{\by}^*}
    \end{align*}
     Next we check for bounded deviations.
     \begin{align*}
	\norm{\epsilon_b} & = \norm{\mathbf{z}_b - \E[\mathbf{z}_b | \mathcal{F}_{b-1}]}\\
	&=\norm{(\widehat{\by}_b-\widehat{\by}^*)-(\mathbb{E}[\widehat{\by}_b-\widehat{\by}^*|\mathcal{F}_{b-1}])}\\
	&= \norm{\widehat{Y_b}-\mathbb{E}[\widehat{Y_b}|\mathcal{F}_{b-1}]}\\
	& = \norm{\frac{\lambda}{b} (\E[\bS_n]-\bS_n)(\by-\Gamma_M(\widetilde{\by}_{b-1})+\mathbb{E}[\Gamma_M(\widetilde{\by}_{b-1})|\mathcal{F}_{b-1}])}\\
	&\leq \frac{\lambda}{b}\norm{\mathbb{E}[\bS_n]-\bS_n}\norm{Y-\Gamma_M(\widetilde{\by}_{b-1})+\mathbb{E}[\Gamma_M(\widetilde{\by}_{b-1})|\mathcal{F}_{b-1}]}
    \end{align*}
    Now, by \cite{zhou2022boulevard}, we know $\E[\bS_n]$ has both row sum and column sum no greater than 1, hence we see that$\norm{\mathbb{E}[\bS_n]}\leq \sqrt{\norm{\mathbb{E}[\bS_n]}_1\cdot \norm{\E[\bS_n]}_{\infty}}\leq \sqrt{1\times n} = \sqrt{n}$. Then, it can be seen that
    \begin{align*}
	\norm{\epsilon_b} & \leq \frac{\lambda}{b}\cdot (1+\sqrt{n})\cdot 2M
    \end{align*}
    Then we can show that $\sum_{b=1}^\infty\norm{\epsilon_b}^2=\sum_{b=1}^\infty O(\frac{1}{b^2})<\infty.$
    
    Similarly we can show both conditions for Algorithm \ref{algorithm1}. Replacing the stepwise difference according to the new update rule gives:
    \begin{align*} 
    \norm{\E[\mathbf{z}_{b+1}\mid\mathcal{F}_b]}
    &= \norm{\E[\widehat{\by}_{b+1,K}-\widehat{\by}_K^*\mid\mathcal{F}_b]}\\
    &= \norm{\frac{b}{b+1}\,\widehat{\by}_{b,K}
       + \frac{1}{b+1}\,\E[\bS_n]\sum_{k=1}^{K}\bigl(\by-\widetilde{\by}_{b,k-1}\bigr)
       - \widehat{\by}_K^*}\\
    &= \norm{\frac{b}{b+1}\,\widehat{\by}_{b,K}
       + \frac{1}{b+1}\,\E[\bS_n]\sum_{k=1}^{K}
         \Bigl(\by - \bigl(\widehat{\by}_{b,K} - \Gamma_M\bigl(\tfrac{1}{b}\sum_{g=1}^b\widehat{\mathbf{t}}_{g,k}\bigr)\bigr)\Bigr)
       - \widehat{\by}_K^*}\\
    &= \norm{\frac{b}{b+1}\,\widehat{\by}_{b,K}
       + \frac{K}{b+1}\,\E[\bS_n]\,\by
       - \frac{K}{b+1}\,\E[\bS_n]\,\widehat{\by}_{b,K}
       + \frac{1}{b+1}\,\E[\bS_n]\sum_{k=1}^K
         \Gamma_M\bigl(\tfrac{1}{b}\sum_{g=1}^b\widehat{\mathbf{t}}_{g,k}\bigr)
       - \widehat{\by}_K^*}\\
    &= \norm{\frac{1}{b+1}\bigl(bI - K\,\E[\bS_n]\bigr)
              \bigl(\widehat{\by}_{b,K}-\widehat{\by}_K^*\bigr)
        + \frac{1}{b+1}\,\E[\bS_n]
          \Bigl(\sum_{k=1}^K\Gamma_M\bigl(\tfrac{1}{b}\sum_{g=1}^b\widehat{\mathbf{t}}_{g,k}\bigr)
                - \widehat{\by}_K^*\Bigr)}\\
    &\le \frac{1}{b+1}\,\norm{b\bI - K\,\E[\bS_n]}\,\norm{\widehat{\by}_{b,K}-\widehat{\by}_K^*}
         + \frac{1}{b+1}\,\norm{\E[\bS_n]}\,
           \norm{\sum_{k=1}^K\Gamma_M\bigl(\tfrac{1}{b}\sum_{g=1}^b\widehat{\mathbf{t}}_{g,k}\bigr)
                 - \widehat{\by}_K^*}\\
    &\le  \frac{1}{b+1}\,\norm{b\bI - K\,\E[\bS_n]}\,\norm{\widehat{\by}_{b,K}-\widehat{\by}_K^*}
         + \frac{1}{b+1}\,\norm{\E[\bS_n]}\,
           \norm{\Gamma_{KM}\bigl(\sum_{k=1}^K\tfrac{1}{b}\sum_{g=1}^b\widehat{\mathbf{t}}_{g,k}\bigr)
                 - \widehat{\by}_K^*}\\
    &\le  \frac{1}{b+1}\,(\norm{b\bI - K\,\E[\bS_n]}+ \norm{\E[\bS_n]})\,\norm{\widehat{\by}_{b,K}-\widehat{\by}_K^*}\\
    &\le  \frac{1}{b+1}\,(b+1-O(\frac{1}{n}))\,\norm{\widehat{\by}_{b,K}-\widehat{\by}_K^*}\\
    &<  \norm{\bz_{b}}
\end{align*}
Hence we have checked the mean contraction condition. Checking the bounded deviation condition works in a similar way for Algorithm \ref{alg:structured-dropout} since it shares the same expected tree kernel with Algorithm \ref{alg:random-dropout} and the tree signal is also bounded by $M$. Hence the stochastic contraction mapping follows.
\end{proof}

\section{Conditional CLT Around KRR Predictions}
\label{app:conditional-clt}

Doing so requires control on the weighting vectors $\r_n(x)$. And we will begin such analysis on giving the rate of the expected structure vector $\k_n(x)$ first. This is because recall that in both algorithms our final weight vector $\r_n(x)$ is a linear map of the voting vector $\k_n(x)$.To begin with, we will first notice that Lemma \ref{lem:kernelrate} implies Lemma \ref{lem:locality}.

\subsection{Bounding the weights}
\begin{lem}
	Suppose the assumptions for Lemma \ref{lem:kernelrate} are satisfied. We have $\left| \sum_{i=1}^n \r^D_{n,i} - \frac{\lambda}{1+\lambda q}\right| = \O{\frac{1}{n}}, \quad
        \left| \sum_{i=1}^n \r^P_{n,i} - 1\right| = \O{\frac{1}{n}}.$
	\label{lem:locality}
\end{lem}

\begin{proof}
    We begin with the analysis on Algorithm \ref{alg:random-dropout}. Consider the expansion
	\begin{align*}
		\left[\frac{1}{\lambda}I + q\bK_n\right]^{-1} = \lambda \sum_{i=0}^{\infty}\left((\lambda q)^{2i}\bK_n^{2i} - (\lambda q)^{2i+1}\bK_n^{2i+1}\right).
	\end{align*}
	We examine the column sums of each of the matrix powers. Start with $K_n^2$,
	$$
	\sum_{i=1}(\bK^2_n)_{i,1} = \sum_{i=1}^n \sum_{j=1}^n (\bK_n)_{i,j}(\bK_n)_{j,1} = \sum_{j=1}^n (\bK_n)_{j,1}\sum_{i=1}^n(\bK_n)_{i,j}.
	$$
	Take subsampling in to consideration, since $K_n$ consists of structure vectors of sample points, hence for some $c>0$, we can bound the row sum:
	$$
	1-\frac{c}{n} \leq \sum_{i=1}^n (\bK_n)_{i,j} \leq 1, \quad i = 1, \dots, n.
	$$
	which also hold true for
	$$
		1-\frac{c}{n} \leq \sum_{j=1}^n (\bK_n)_{j,1} \leq 1, \quad i = 1, \dots, n.
	$$
	Given $K_n$ is nonnegative, multiply the inequalities above and notice it is actually the row sum of $K_n^2$, we have
	$$
	\left(1-\frac{c}{n}\right)^2 \leq \sum_{i=1}(\bK^2_n)_{i,1} = \sum_{j=1}^n (\bK_n)_{j,1}\sum_{i=1}^n(\bK_n)_{i,j} \leq 1.
	$$
	Repeating the same procedure above yields
	$$
	\left(1-\frac{c}{n}\right)^m \leq \sum_{i=1}(\bK^m_n)_{i,1}\leq 1.
	$$
	Therefore,
	\begin{align*}
		\lambda \left(\frac{1}{1-\lambda^2q^2(1-\frac{c}{n})^2} - \frac{\lambda q}{1-\lambda^2q^2} \right) & \leq \sum_{j=1}^n \left[\frac{1}{\lambda}I + q\bK_n\right]^{-1}_{j,1} \\
		& = \lambda \left(\sum_{i=0}^{\infty}(\lambda q)^{2i}(\bK_n^{2i})_{j,1} - (\lambda q)^{2i+1}(\bK_n^{2i+1})_{j,1} \right) \\
		& \leq \lambda \left(\frac{1}{1-\lambda^2q^2} - \frac{\lambda q}{1-\lambda^2q^2(1-\frac{c}{n})^2} \right),
	\end{align*}
	where both the LHS and RHS reduce to $\frac{\lambda}{1+\lambda q} + \O{\frac{1}{n}}$. So is true for any column sum of $\left[\frac{1}{\lambda}I + q\bK_n\right]^{-1}$.And we know that $\k_n$ is nonnegative and $1 - \norm{\k_n}_1 = \O{\frac{1}{n}}$ and $\k_n$ simply reweights the columns. Hence we prove the convergence of weights $\sum_{i=1}^n\r_{n,i}$ for Algorithm \ref{alg:random-dropout}.

    The analysis for Algorithm \ref{alg:structured-dropout} is similar. We simply substitute $\lambda = K, q =\frac{K-1}{K}$. This will give the bound:
    \begin{align*}
        K \left(\frac{1}{1-(K-1)^2(1-\frac{c}{n})^2} - \frac{K-1}{1-(K-1)^2} \right) & \leq \sum_{j=1}^n \left[\frac{1}{K}I + \frac{K-1}{K}\bK_n\right]^{-1}_{j,1} \\
        & = K \left(\sum_{i=0}^{\infty}(K-1)^{2i}(\bK_n^{2i})_{j,1} - (K-1)^{2i+1}(\bK_n^{2i+1})_{j,1} \right) \\
        & \leq K \left(\frac{1}{1-(K-1)^2} - \frac{K-1}{1-(K-1)^2(1-\frac{c}{n})^2} \right),
	\end{align*}
    We can reduce both sides of the inequalities to $1+O(\frac{1}{n})$ similarly. That proves the converging weight of Algorithm \ref{alg:structured-dropout}.
\end{proof}

Intuitively, knowing that the weights sum to (almost) a constant tells us the total “mass” of our weighting vector is fixed. If we have leaf size assumptions such that no single weight can be too large, the mass must be spread out over many small pieces. Whenever you distribute a fixed amount of weight across many entries, the Euclidean length of the vector necessarily shrinks. In other words, a near‐unit \(\ell_1\) norm plus a bound on the largest coordinate forces the \(\ell_2\) norm to be small. We control the rate of $\r_n^{D,P}$ in Lemma \ref{lem:rateofr}. To do so, we need to both upper bound and lower bound $\norm{\k_n}$, which is helped by Assumption \ref{aspt:ub-minimal-leaf-size}, implying $\inf_{A\in q\in Q_n}\sum_{i=1}^n\mathbbm{1}(\bx_i\in A)=\Ome{n^{\frac{1}{d+1}}}$.

\subsection{Proof for Lemma \ref{lem:rateofr}}
We will make use of the $\Ome{n^{\frac{1}{d+1}}}$ in Assumption \ref{aspt:median-trees} to control $\norm{\k_n}$ and $\norm{\r_n}$. The first bound in the $\max$ operator is vital for our unconditional CLT analysis Lemma \ref{lem:expdecayp}. We will first show Lemma \ref{lem:rateofr} which is sufficient for us to get a conditional CLT.

\begin{proof}
    To bound $\norm{\k_n}$, recall:
	$$
	\k_{nj}=\mathbb{E}[\s_{n,j}(x)]=\mathbb{E}\left[\frac{\mathbbm{1}(x_j\in A)}{\sum_{j=1}^{n}\mathbbm{1}(x_j\in A)}\right], x\in A
	$$
	encodes the expected influence of point $x_j$ on a point of interest $x$ among other $n$ points. Then the condition $$ \inf_{A \in q \in Q_n} \sum_{i = 1}^n \mathbbm{1}(x_i \in A)  \geq \Ome{n^{\frac{1}{d+1}} } $$
	implies that $\k_{nj} = \O{n^{-\frac{1}{d+1}}}$, since we can't have total weights of more than $O(1)$ in a leaf. By Lemma \ref{lem:kernelrate}, $\norm{\k_n}_1 \leq 1$,
	$$
	\norm{\k_n} \leq \sqrt{\norm{\k_n}_1 \norm{\k_n}_{\infty}} = \O{n^{-\frac{1}{2}\frac{1}{d+1}}}.
	$$
	By assetion $\left| B_n \right|  = \Ome{ n \cdot d_n^d }$, there are at most
	$$
	\Ome{n \cdot d_n^d} = \Ome{n^{\frac{1}{d+1}}}
	$$
	$k_{nj}$'s that are positive. Equivalently, we can have the magnitude of each non-zero $\k_{nj}$ is lower bounded by $\Ome{n^{-\frac{1}{d+1}}}$. This holds also because the total weights can't be larger than $O(1)$. Since $\norm{\k_n}_1 = 1-O(n^{-1})$, $$
	\norm{\k_{n}} = \Ome{ \sqrt{\left(n^{-\frac{1}{d+1}} \right)^2 \cdot n^{\frac{1}{d+1}}}} =  \Ome{n^{-\frac{1}{2}\frac{1}{d+1}}}.
	$$
    To check the rate of $\norm{\r_n}$, since $r_n$ is a mapped from $k_n$ by the KRR matrix. For Algorithm \ref{alg:random-dropout} it's $[\frac{1}{\lambda}\bI + q\bK_n]$ and for Algorithm \ref{alg:structured-dropout} is  $[\frac{1}{K}I+\frac{K-1}{K}K_n]^{-1}$ and we know 
    $$
    \frac{\lambda}{1+\lambda q} \leq eigen\left(\left[\frac{1}{\lambda}\bI + q\bK_n\right]^{-1}\right) \leq \lambda,\quad 1\leq eigen\left(\left[\frac{1}{K}\bI+\frac{K-1}{K}\bK_n\right]^{-1}\right)\leq K
    $$
    Since we can fix $K$ to be a constant, $\norm{\r_n}$ will enjoy similar rates as $\norm{\k_n}$. 
\end{proof}

This rate complies with the Lindeberg-Feller conditions, hence we can establish the conditional CLT.

\subsection{Conditional Asymptotic Normality on Training Set}
\begin{thm}[Conditional Asymptotic Normality for BRAT-D and BRAT-P Predictions]
	\label{thm:fixed}
	For any $\bx \in [0,1]^d$, write $f(\bX_n) = (f(\bx_1),\dots, f(\bx_n))^\top $. Then under Assumptions \ref{aspt:svi}, \ref{aspt:non-adaptivity}, \ref{aspt:leaf-diameter} and assumptions in Lemma \ref{lem:rateofr}, we have
	$$
	\frac{\widehat{\f}^{D,P}_n(\bx) - \langle{\r^{D,P}_n}, \f(\bX_n)\rangle}{\norm{\r_n^{D,P}}} \xlongrightarrow{d} \mathcal{N}(0,\sigma^2).
	$$
\end{thm}
\begin{proof}
    For the ease of notation, since $\norm{\r_n}$ yielded by both Algorithm \ref{alg:random-dropout} and Algorithm \ref{alg:structured-dropout} shares the same rate on $n$, we will suppress the notation of $D,P$ in this proof, and denote the kernel ridge regression matrix respectively with $\mathsf{KRR}^D$ and $\mathsf{KRR}^P$. Also define $c^D=\frac{\lambda}{1+\lambda q}, c^P=1$.

    Write
	\begin{align*}
		\widehat{\f}_n(x) - \r_n^T\f(X_n) =\r_n^T \vec{\epsilon}_n.
	\end{align*}
	To obtain a CLT we check the Lindeberg-Feller condition of $\r_n^T \vec{\epsilon}_n$, i.e. for any $\delta > 0$,
	$$\
	\lim_n \frac{1}{\norm{\r_n}^2 \sigma^2}\sum_{i=1}^n  \E \left[(\r_{ni} \epsilon_i)^2 \mathbbm{1}(|\r_{ni}\epsilon_i| > \delta \norm{\r_n}\sigma_) \right] = 0.
	$$
	Since $\norm{\k_n}_{\infty} = O\left(n^{-\frac{1}{d+2}}\right)$ and $\mathsf{KRR}^{D,P}$ having row sums of $C^{D,P}+\O{n^{-1}}$, we have
	$$
	\norm{\r_n}_{\infty} \leq \norm{\k_n}_{\infty} \cdot \norm{\mathsf{KRR}^{D,P}}_1 = O\left(n^{-\frac{1}{d+1}}\right).
	$$
	Furthermore, since $\norm{\r_n} = \Theta(n^{-\frac{1}{2}\frac{1}{d+1}})$, we get
	$$
	\frac{\norm{\r_n}_{\infty}}{\norm{\r_n}} = O\left( n^{-\frac{1}{2}\frac{1}{d+1}}\right),
	$$
	This allows us to check out the Lindeberg-Feller conditions:
	\begin{align*}
		\sum_{i=1}^n  \E \left[(\r_{ni} \epsilon_i)^2 \mathbbm{1}(|\r_{ni}\epsilon_i| > \delta \norm{\r_n}\sigma) \right]
		& \leq \sum_{i=1}^n \r_{ni}^2 \sqrt{\E[\epsilon_i^4]\cdot \E[\mathbbm{1}(|\r_{ni}\epsilon_i| > \delta \norm{\r_n}\sigma^2]}  \\
		& \leq \sum_{i=1}^n \r_{ni}^2 \sqrt{\E[\epsilon_i^4]} \cdot \sqrt{P\left( |\epsilon_i| \geq \frac{\delta \| \r_n \| \sigma}{\r_{ni}} \right)} \\
		& \leq \sum_{i=1}^n \r_{ni}^2 \sqrt{\E[\epsilon_i^4]} \sqrt{2 \exp\left(-\frac{1}{2\sigma^2}\cdot \left(\frac{\delta \| \r_n \| \sigma}{\r_{ni}} \right)^2\right)} \\
		& \leq \norm{\r_n}^2 \exp \left( - \O{n^{\frac{1}{d+1}}}\right) \longrightarrow 0,
	\end{align*}
	Since $\epsilon$ is sub-Gaussian noise, the concerntraion bound holds by definition of $\epsilon$.
\end{proof}
\section{Extension to Random Design and Exponential Locality}
\label{app:random-clt}
Now we move on to the discussions in which our input to the model is no longer restricted only to the training set $\X_n$, but to any $\bx\in[0,1]^d$. To do so, we begin by defining a new probability space and, of course, a new probability measure. This can be uniquely designed by a Komolgorov extension theorem.

Write the coordinate projection $\Pi_{n}=(\pi_n(\X), \pi_n(\vec{\epsilon}))$ as the finite-dimensional projection containing the first $n$ samples. Under this framework, we can obtain the corresponding expected structure vector $\k_n$, expected structure matrix (kernel matrix) $K_n$ and the standardized prediction error $\rho_n(\X, \vec{\epsilon})$, given by
\[
\rho_n(\X, \epsilon) = \frac{\widehat{\f}^{D,P}_n(\bx; \Pi_n) - \langle\k_n^{\top}(\bx; \Pi_n)\mathsf{KRR}^{D,P}(\Pi_n),\f(\Pi_n)\rangle}{\norm{\k_n(\bx;\Pi_n)^\top \mathsf{KRR}^{D,P}\Pi_n}}
\]
Here, $\rho_n$ denotes the prediction error after using $n$ random samples incorporating both the dropout mechanism and the sample randomness.

To prove our desired CLT under random design, we need to first introduce a lemma that allows us to propagate the normality from fixed sample to random sample cases. This is taken care of by Lemma \ref{lem:fixtorandom}, which allows us to claim our CLTs if we verify it being correct for $a.s. \forall \bx\in[0,1]^d$.  Then, we seek an almost sure validation of the assumptions of Theorem \ref{thm:fixed} under the random sample sequence $\Pi_{n}=(\pi_n(\X), \pi_n(\vec{\epsilon}))$, which is helped by our Assumption \ref{aspt:restricted-tree-space}. Formally, it is as shown below:
\begin{lem}
	\label{lem:random}
	Suppose we have random sample $\{(\bx_i, \by_i)\}_{i=1}^n\subset [0,1]^d\times \R$ for each $n$. If for any small $\alpha, \nu > 0$ s.t.
	$$
	|Q_n| = \O{\frac{1}{n} \exp\left( \frac{1}{2} n^{\frac{1}{d+2}-\nu} - n^{\alpha}\right)},
	$$
	then
	$$
	\frac{\widehat{\f}_n^{D,P}(\bx) - \langle\r_n^{D,P},\f(\X_n)\rangle}{\norm{\r_n^{D,P}}} \xlongrightarrow{d} \mathcal{N}(0,\sigma^2).
	$$
\end{lem}
\label{sec:lemma:random}
\begin{proof}
	We will direct the reader to \cite{zhou2022boulevard}. Despite the introduction of stochasticity in dropout, extending the CLT of \ref{thm:fixed} does not require additional effort. The intuition for the proof is as follows: By restricting the tree space, one bounds the probability of the assumptions being violated to be summable, and per Borel-Cantelli will take place almost surely. Tree space cardinality here gives a nice union probability bound to achieve that.
\end{proof}

We define the following notation for later use: For any $n$-vector $v$ and an index set $D$, denote
$$
v \proj{D} = \begin{bmatrix}
	v_1 \cdot \mathbbm{1}(1 \in D) \\
	\vdots\\
	v_n \cdot \mathbbm{1}(n \in D)
\end{bmatrix}.$$
This implies the decomposition that
$v = v\proj{D} + v\proj{D^c}$. 

We are now almost ready to prove the main theorem. However, to check Lindeberg-Feller conditions under random design, one needs to establish exponential decay of points that sit far away from the point of interest. That is, $\norm{\r_n|_{D^c}}$ for a well chosen ball centered at $\bx$ needs to vanish to 0 rather quickly. And we prove it in the following lemma. To control this for Algorithm \ref{alg:random-dropout}, we have the following lemma:
\begin{lem}
	\label{lemma:expdecayd}
	Given sample $\{(\bx_i, \by_i)\}_{i=1}^n$ and a point of interest $\bx$, set
	$
	l_n = \frac{\log n}{-\log \lambda q} = c_1\log n, c_1=\log \frac{1}{\lambda q}
	$
	and define index set $D_n = \{i : |\bx_i-\bx| \leq l_n\cdot d_n   \}$,
	then
    $$
    \norm{\r_n^D\proj{D^c_n}}_1 = O\left(\frac{1}{n}\right).
    $$
\end{lem}
\begin{proof}
	Let's first write out the expression of the desired peripheral points:
	$$
	\norm{\r_n^D\proj{D_n^c}}_1 =  \sum_{|\bx - \bx_i|>l_n \cdot  d_n}|\r^D_{ni}|
	$$
	
	Recall:
	$$
	\k_n^T=\mathbb{E}[\s_n(\bx)] \to \k_{nj}=\mathbb{E}[\s_{n,j}(\bx)]=\mathbb{E}[\frac{\mathbbm{1}(\bx_j\in A)}{\sum_{j=1}^{n}\mathbbm{1}(\bx_j\in A)}], \bx\in A
	$$
	This term would be zero if two points are not even in the same leaf, i.e. $|\bx-\bx_j|>d_n$. Hence the sum inside would be simplified to
	$$
	\sum_{|\bx - \bx_i|>l_n \cdot  d_n} \left|\sum_{|\bx_j-\bx|\leq d_n} \k_{nj} \left[\frac{1}{\lambda}\bI + q\bK_n\right]^{-1}_{j,i} \right|
	$$
	And also notice, by decreasing $l_nd_n$ to $(l_n-1)d_n$, we are actually allowing more peripheral points outside the local region $D_n$ and hence we have the inequality at line 4 below. Completely. it is given as:
	\begin{align*}
		\norm{\r^D_n\proj{D_n^c}}_1 =  \sum_{|\bx - \bx_i|>l_n \cdot  d_n}|\r^D_{ni}|
		= &  \sum_{|\bx - \bx_i|>l_n \cdot  d_n} \left|\sum_j \k_{nj} \left[\frac{1}{\lambda}\bI + q\bK_n\right]^{-1}_{j,i} \right| \\
		= & \sum_{|\bx - \bx_i|>l_n \cdot  d_n} \left|\sum_{|\bx - \bx_j| \leq d_n} \k_{nj} \left[\frac{1}{\lambda}\bI + q\bK_n\right]^{-1}_{j,i} \right| \\
		\leq & \sum_{|\bx - \bx_j| \leq  d_n} \k_{nj} \sum_{|\bx - \bx_i|>l_n \cdot  d_n} \left| \left[\frac{1}{\lambda}\bI + q\bK_n\right]^{-1}_{j,i} \right| \\
		\leq &  \sum_{|\bx - \bx_j| \leq d_n} \k_{nj} \sum_{|\bx_i - \bx_j| > (l_n-1)\cdot d_n}\left| \left[\frac{1}{\lambda}\bI + q\bK_n\right]^{-1}_{j,i} \right| \\
		\leq & \sum_{|\bx - \bx_j| \leq  d_n} \k_{nj} \sum_{|\bx_i - \bx_j| > (l_n-1)\cdot d_n} \lambda \sum_{l=l_n}^{\infty} (\lambda q)^l [{\bK_n^l}]_{j,i} \\
		\leq & \sum_{|\bx - \bx_j| \leq  d_n} \k_{nj} \sum_{l=l_n}^{\infty} (\lambda q)^{l+1} \\
		\leq & \lambda \sum_{l=l_n}^{\infty} (\lambda q)^{l} =  \frac{\lambda}{1-\lambda q}\frac{1}{n}.
	\end{align*}
	The lower bound in the sum of the powers of $\bK_n$ follows this argument. To compute higher powers of $\bK_n$, consider how entries of $K_n^l$ are calculated:
	\[
	\bK_n^l[i, j] = \sum_{k_1, k_2, \dots, k_{l-1}} \bK_n[i, k_1] \bK_n[k_1, k_2] \dots \bK_n[k_{l-1}, j].
	\]
	For $\bK_n^l[i, j] \neq 0$, there must exist a chain of intermediate points $\{k_1, k_2, \dots, k_{l-1}\}$ such that:
	\[
	|\bx_i - \bx_{k_1}| \leq d_n, \quad |\bx_{k_1} - \bx_{k_2}| \leq d_n, \quad \dots, \quad |\bx_{k_{l-1}} - \bx_j| \leq d_n.
	\]
	By summing the distances and by a triangular inequality, the locality condition propagates:
	\[
	|\bx_i - \bx_j| \leq l \cdot d_n,
	\]
	where $l$ is the number of steps (or multiplications) required to connect $x_i$ and $x_j$ through the intermediate points. Now notice the summing condition $|x_i-x_j|>(l_n-1)d_n$ requires us to have such chain that is longer than $l_n-1$ at least. Hence the kernel's power $\bK_n^l$ would be 0 for any $l\leq l_n-1$, since such chain doesn't exist at all. 
\end{proof}
To control weights from distant points for Algorithm \ref{alg:structured-dropout} is a bit more complicated. We need to make use of Assumption \ref{aspt:median-trees}. This assumption is inspired by the definition in \cite{athey2018generalizedrandomforests}, where we state below.

\begin{defn}[\(\alpha\)‐regular tree predictor]
\label{def:alpha_regular}
A tree predictor grown by recursive partitioning is called \emph{\(\alpha\)‐regular} for some \(\alpha>0\) if either
\begin{enumerate}
  \item \textbf{(Standard case)} each split leaves at least a fraction \(\alpha\) of the available training examples on \emph{each} side of the split, and furthermore the trees are fully grown to depth \(k\) for some \(k\in\mathbb{N}\).  Equivalently, every terminal node contains between \(k\) and \(2k-1\) observations;
  \item \textbf{(SVI case)} It satisfies part (a) when applied to the sample used to fit the leaf value. In this case this will be our response vector $Y$.
\end{enumerate}
\end{defn}

\begin{lem}
    \label{lem:rateofP-alpha}
    If a tree is $\frac{1}{2}$-regular with the terminal leaf size being lower bounded by $k$, where $k$ is a constant, then the operator norm of $\mathbf{P}$ is of rate $O(1)$.
\end{lem}

\begin{proof}
    Define $\alpha_*=\max\{\alpha, 1-\alpha\}$. By definition of $\alpha$-regularity, the probability of point $i,j$ falling into the same leaf after one split would be upper bounded by $\alpha_*$. And since we have the minimal terminal leaf size requirement, for any leaf with depth $d$, it should comply with:
    \[
    \begin{cases}
    k \le \alpha_*^dn \le 2k-1,\\
    k \le (1-\alpha_*)^dn \le 2k-1.
    \end{cases}
    \]

which translates into $\frac{1}{\ln(1-\alpha_*)}\ln(\frac{2k-1}{n})\leq d \leq \frac{1}{\ln(\alpha_*)}\ln(\frac{k}{n})$.

Notice that the probability of two points sharing a leaf is the probability that they survive all splits down the tree. Then we can bound the collision probability by:
\begin{align*}
    p_{i,j} &\leq \alpha_*^{d_{min}}\\
    &= \alpha_*^{\frac{1}{\ln(1-\alpha_*)}\ln(\frac{2k-1}{n})}\\
    &=(\frac{2k-1}{n})^{\frac{\ln \alpha_*}{\ln (1-\alpha_*)}}\\
    &=(\frac{2k-1}{n})^{\log_{1-\alpha_*}\alpha_*}\\
    &=\frac{2k-1}{n}
\end{align*}
Then $\sum_{i}p_{i,j} = O(1)$
\end{proof}

Well-behaved leaves exclude non-decaying distant weights. Formally, we can prove the lemma below:
\begin{lem}
\label{lem:expdecayp}
    Given sample $\{(\bx_i, \by_i)\}_{i=1}^n$ and a point of interest $\bx$, set
    $
    l_n = -\frac{\ln n}{\ln (\xi/2)} = c\ln n
    $
    and define index set $D_n = \{i : |\bx_i-\bx| \leq l_n\cdot d_n   \}$.
    then
    $$
    \norm{r_n^P\proj{D^c_n}}_1 = O\left(\frac{1}{n}\right).
    $$
\end{lem}
    \begin{proof}
    We will use the same matrix expansion technique we used in Lemma\ref{lemma:expdecayd}. By analogy,
    \begin{align*}
        \norm{\r_n^P\proj{D_n^c}}_1 =  \sum_{|\bx - \bx_i|>l_n \cdot  d_n}|\r_{ni}^P|
        = &  \sum_{|\bx - \bx_i|>l_n \cdot  d_n} \left|\sum_j \k_{nj} \left[\frac{1}{K}\bI + \frac{K-1}{K}\bK_n\right]^{-1}_{j,i} \right| \\
        = & \sum_{|\bx - \bx_i|>l_n \cdot  d_n} \left|\sum_{|\bx - \bx_j| \leq d_n} \k_{nj} \left[\frac{1}{K}\bI + \frac{K-1}{K}\bK_n\right]^{-1}_{j,i} \right| \\
        \leq & \sum_{|\bx - \bx_j| \leq  d_n} \k_{nj} \sum_{|\bx - \bx_i|>l_n \cdot  d_n} \left| \left[\frac{1}{K}\bI + \frac{K-1}{K}\bK_n\right]^{-1}_{j,i} \right| \\
        \leq &  \sum_{|\bx - \bx_j| \leq d_n} \k_{nj} \sum_{|\bx_i - \bx_j| > (l_n-1)\cdot d_n}\left| \left[\frac{1}{K}\bI + \frac{K-1}{K}\bK_n\right]^{-1}_{j,i} \right| \\
        = & \sum_{|\bx - \bx_j| \leq  d_n} \k_{nj} \sum_{|\bx_i - \bx_j| > (l_n-1)\cdot d_n} K \Big|\sum_{l=l_n}^{\infty} (-1)^l(K-1)^l [{\bK_n^l}]_{j,i}\Big| \\
        \leq & \sum_{|\bx - \bx_j| \leq  d_n} \k_{nj} K\sum_{l\geq l_n}(K-1)^l\sum_{|\bx_i - \bx_j| > (l_n-1)\cdot d_n}|[\bK_n^l]_{i,j}|\\
        \leq & \sum_{|\bx - \bx_j| \leq  d_n} \k_{nj} K\sum_{l\geq l_n}(K-1)^l\sum_{i=1}^n|[\bK_n^l]_{i,j}|\\
        \leq & \sum_{|\bx - \bx_j| \leq  d_n} \k_{nj} K\sum_{l\geq l_n}(K-1)^l\norm{\bK_n^l}_{\infty}\\
        \leq & \sum_{|\bx - \bx_j| \leq  d_n} \k_{nj} K\sum_{l\geq l_n}(K-1)^l\norm{\bK_n}_{\infty}^l\\
    \end{align*}
    Denote the probability of any point $\bx_i, \bx_j$ sharing the same leaf as $p_{ij}$. By definition of the kernel matrix, we can write each element as:
    \begin{align*}
        [\bK_n]_{i,j}&=\E_{q,w}\Big[[\bS_n]_{i,j}\Big]\\
        &=\E_{w}\Big[[\bS_n]_{i,j}\Big|\bx_i, \bx_j \text{ are in the same leaf}\Big]p_{ij}\\
        &=\mathbb{E}[\frac{1}{X}]p_{i,j}
    \end{align*}
    where $X\sim \text{Binomial}([m,M], \xi)$, with $m$ being the minimal number of points in a leaf and $M$ being the maximal number of points within a leaf. Since $X=m,\cdots, M$, it follows that $\frac{1}{X}\leq \frac{1}{m}$ almost surely and hence $\mathbb{E}[\frac{1}{X}]\leq \frac{1}{m\xi}$. Then by assumption $m = \max\{\Ome{n^{\frac{1}{d+2}}},\frac{K-1}{\xi}\}$, it follows that $\mathbb{E}[\frac{1}{X}]\leq \frac{\xi}{K-1}$. Hence we have
    \begin{align*}
        \norm{\r_n^P\proj{D_n^c}}_1\leq & \sum_{|\bx - \bx_j| \leq  d_n} \k_{nj} K\sum_{l\geq l_n}(K-1)^l\norm{\bK_n}_{\infty}^l\\
        \leq & \sum_{|\bx - \bx_j| \leq  d_n} \k_{nj} K\sum_{l\geq l_n}(K-1)^l\frac{\xi^l}{2^l(K-1)^l}\norm{\mathbf{P}}_{\infty}\\
        \leq & \sum_{|\bx - \bx_j| \leq  d_n} \k_{nj} K\sum_{l\geq l_n}\frac{\xi^l}{2^l}o(1)\\
        \leq & (\frac{\xi}{2})^{l_n}o(1)\\
        = & o(\frac{1}{n})
    \end{align*}
\end{proof}

\section{Proof for Theorem \ref{thm:main}}
\label{app:main-theorem}
Combining all discussions above, we present the main result Theorem \ref{thm:main} of this paper.
\begin{proof}
	Define $c^{D}=\frac{\lambda}{1+\lambda q}$ and $c^P=1$. To begin with, we will show that for $\forall \bx \in [0,1]^d$,
	$$
	\frac{\langle\r_n^{D,P},f(\bX_n)\rangle - c^{D,P}\f(\bx)}{\norm{\r^{D,P}_n}} \xlongrightarrow{\mathbb{P}} 0.
	$$
	Let $\Delta^{D,P} = c^{D,P} - \sum_{i=1}^n \r_{n,i}^{D,P} = \O{n^{-1}}$ and $\widetilde{\f}(\bx) = (\f(\bx),\dots, \f(\bx))^\top$ an $n$-vector. We consider the following decomposition:
	\begin{align*}
		\frac{\langle\r_n^{D,P},f(\bX_n)\rangle - c^{D,P}\f(\bx)}{\norm{\r^{D,P}_n}} & =  \frac{\langle\r_n^{D,P}, \f(\bX_n) - \widetilde{\f}(\bx)\rangle}{\norm{\r_n^{D,P}}} - \frac{\Delta^{D,P} \cdot \f(\bx)}{\norm{\r_n^{D,P}}} \\
		= & - \frac{\Delta \cdot \f(\bx)}{\norm{\r_n^{D,P}}} + \frac{\langle \r_n^{D,P}\proj{D_n}, [\f(\X_n) - \widetilde{\f}(\bx)]\proj{D_n}\rangle}{\norm{\r_n^{D,P}}} + \frac{\langle\r_n^{D,P}\proj{D_n^c}, [\f(\X_n) - \widetilde{\f}(\bx)]\proj{D_n^c}\rangle}{\norm{\r_n^{D,P}}}.
	\end{align*}
	By \ref{lem:random}, we have
	$$
	\norm{\k_n}=\Theta_p(n^{-\frac{1}{d+1}}), \norm{\r_n^{D,P}} = \Theta_p(n^{-\frac{1}{2}\frac{1}{d+1}}).
	$$
	Notice that
	$$
	\left| \langle\r_n^{D,P}\proj{D_n^c}, [\f(\X_n) - \widetilde{\f}(\bx)]\proj{D_n^c}\rangle\right| \leq \norm{\r_n^{D,P}\proj{D_n^c}}_1 \cdot \norm{[\f(\X_n) - \widetilde{\f}(\bx)]\proj{D_n^c}}_{\infty} = \Op{\frac{1}{n} \cdot 2M_\f} = \Op{n^{-1}}.
	$$
	Hence we have the second term
	$$
	\frac{\langle \r_n^{D,P}\proj{D_n}, [\f(\X_n) - \widetilde{\f}(\bx)]\proj{D_n}\rangle}{\norm{\r_n^{D,P}}} \xlongrightarrow{\mathbb{P}} 0.
	$$
	And by an argument earlier $| \Delta^{D,P} | = \O{n^{-1}}$, thus
	$$
	\frac{\Delta \cdot \f(\bx)}{\norm{\r_n}} \xlongrightarrow{\mathbb{P}} 0.
	$$
	Meanwhile, we can show similarly $|D_n| = \Omep{n \cdot (l_n\cdot d_n)^d)}$ a.s.. And recall our underlying function is $\alpha$-Lipischitz, therefore
	\begin{align*}
		\frac{\langle\r_n^{D,P}\proj{D_n^c}, [\f(\X_n) - \widetilde{\f}(\bx)]\proj{D_n^c}\rangle}{\norm{\r_n^{D,P}}} & \leq \frac{\norm{\r_n^{D,P}\proj{D_n}}\cdot\norm{[\f(\X_n) - \widetilde{\f}(\bx)]\proj{D_n}}}{\norm{\r_n^{D,P}}} \\
		& \leq \norm{[\f(\X_n) - \widetilde{\f}(\bx)]\proj{D_n}} \\
		& = \Op{\sqrt{n\cdot (l_nd_n)^d\cdot (l_nd_n\cdot{\alpha})^2}} \\
		& = \Op{\sqrt{n\cdot (\log n)^{d+2} \cdot d_n^{d+2}}} \\
		& = \Op{\sqrt{n\cdot (\log n)^{d+2} \cdot n^{-\frac{d+2}{d+1}}}}\\
		& = \Op{(\log n)^{\frac{d+2}{2}} n^{-\frac{1}{2}\frac{1}{d+1}}}.
	\end{align*}
	Therefore
	$$
	\frac{\langle\r_n^{D,P},\f(\bX_n)\rangle - c^{D,P}\f(\bx)}{\norm{\r^{D,P}_n}} \xlongrightarrow{\mathbb{P}} 0.
	$$
	Slutsky's Theorem then yields
	$$
	\frac{\widehat{\f}^{D,P}_n(\bx) - c^{D,P}\f(\bx)}{\norm{\r_n^{D,P}}} = \frac{\widehat{\f}_n(\bx) - \langle\r_n^{D,P},\f(\X_n)\rangle}{\norm{\r_n^{D,P}}} + \frac{\langle\r_n^{D,P},\f(\bX_n)\rangle - c^{D,P}\f(\bx)}{\norm{\r^{D,P}_n}}\xlongrightarrow{d} \mathcal{N}(0,\sigma^2_{\epsilon}).
	$$
\end{proof}

\section{Coverage Rates and Risk Bounds Guarantees}
\label{app:statistics}
In this section we provide the proofs for coverage rates and risk bounds guarantees we gave in \ref{sec:theory}. 

\subsection{Proof for Corollary \ref{cor:asymptotic-risk}}
\begin{proof}
    The corollary holds by noticing that the variance of the scaled prediction of Algorithm \ref{alg:random-dropout} is 
    $( \frac{1+\lambda q}{\lambda})^2\norm{\r_n^D}^2\sigma^2$. Plugging in the rate of $\norm{\r_n^D}$ by lemma \ref{lem:rateofr} yields the risk bound. Similarly this can be shown for Algorithm \ref{alg:structured-dropout}.
\end{proof}

\subsection{Proof for Corollary \ref{cor:cov-rate}}
\begin{proof}
    We will show this convergence for the confidence interval Algorithm \ref{alg:random-dropout}. For Algorithm \ref{alg:structured-dropout}, replacing the asymptotic variance and the final rescale constant will follow the same proof. 

    Denote the staged empirical coverage rate of the built-in confidence interval at any point $\bx\in[0,1]^d$ as
    \[
    1-\hat{\alpha}^{\text{CI}}_{n,b}(\bx)=\frac{1}{n}\sum_{i=1}^n\mathbbm{1}\Big[\f(\bx)\in \big[(c^D\widehat{\f}_b(\bx)-c^D\norm{\r_n^D}z_{1-\alpha/2}\sigma,c^D\widehat{\f}_b(\bx)+c^D\norm{\r_n^D}z_{1-\alpha/2}\sigma\big]\Big]
    \]
    Notice that with $b\to\infty$ we have $\widehat{\f}_b\overset{a.s.}{\to}\widehat{f}_{\infty}$. Hence, sending $b$ to infinity and we drop the subscript for $b$ to denote the coverage at infinite epoch. 
    \[
    1-\hat{\alpha}^{\text{CI}}_{n}(\bx)=\frac{1}{n}\sum_{i=1}^n\mathbbm{1}\Big[\f(\bx)\in \big[(c^D\widehat{\f}_\infty(\bx)-c^D\norm{\r_n^D}z_{1-\alpha/2}\sigma,c^D\widehat{\f}_\infty(\bx)+c^D\norm{\r_n^D}z_{1-\alpha/2}\sigma\big]\Big]
    \]
    
    which is equivalent to

    \[
    1-\hat{\alpha}^{\text{CI}}_{n}(\bx)=\frac{1}{n}\sum_{i=1}^n\mathbbm{1}\Big[\frac{\widehat{\f}_\infty(\bx)-c^D\f(\bx)}{\norm{r_n}\sigma}\in \big[-z_{1-\alpha/2},z_{1-\alpha/2}\big]\Big]
    \]

    Since the indicator function is always finite, by SLLN:
    \begin{align*}
    1-\hat{\alpha}^{\text{CI}}_{n}(\bx)&\overset{a.s.}{\to}\mathbb{E}\bigg[\mathbbm{1}\Big[\frac{\widehat{\f}_\infty(\bx)-c^D\f(\bx)}{\norm{r_n}\sigma}\in \big[-z_{1-\alpha/2},z_{1-\alpha/2}\big]\Big]\bigg]\\
    &=\mathbb{P}\Big[\frac{\widehat{\f}_\infty(\bx)-c^D\f(\bx)}{\norm{r_n}\sigma}\in \big[-z_{1-\alpha/2},z_{1-\alpha/2}\big]\Big]\\
    &=1-\alpha
    \end{align*}
    
    For prediction intervals, we observe data according to the model
    \[
      \by \;=\; \f(\bx) \;+\; \varepsilon,\qquad
      \varepsilon\sim\mathcal{N}(0,\sigma^2),
    \]
    and our infinite‐epoch estimator satisfies the CLT
    \[
      \widehat \f_\infty(\bx)
      \;\sim\;
      \mathcal{N}\!\bigl(c^D\,\f(\bx),\,\|\r_n\|^2\,\sigma^2\bigr)
      \quad\text{ as }\quad b\to\infty.
    \]
    Since \(\widehat f_\infty(\bx)\) and the new noise \(\varepsilon\) are independent Gaussians, their sum
    \begin{align*}
    \by &= \f(\bx)-c^D\widehat{\f}_\infty(\bx) + c^D\widehat{\f}_\infty(\bx) + \epsilon\\
    &=\mathcal{N}(c^D\widehat{\f}_\infty(\bx), (c^D\norm{\r_n^D})^2\sigma^2 + \sigma^2)
    \end{align*}
    And we produce prediction interval with the empirical coverage rate evaluated at:
    \begin{align*}
    &\qquad 1-\hat{\alpha}^{\text{PI}}_{n,b}(\bx)\\
    &=\frac{1}{n}\sum_{i=1}^n\mathbbm{1}\Big[\f(\bx)\in \big[(c^D\widehat{\f}_b(\bx)-(c^D\norm{\r_n^D}+1)z_{1-\alpha/2}\sigma,c^D\widehat{\f}_b(\bx)+(c^D\norm{\r_n^D}+1)z_{1-\alpha/2}\sigma\big]\Big]
    \end{align*}
    By similar argument above, sending $n,b\to \infty$ gives us almost sure convergence.

    Regarding the reproduction interval, since we will have two independent model following the CLT in Theorem \ref{thm:main}, their difference will also follow a CLT centered at zero with a variance inflated by a factor of 2 by independence. So we construct the reproduction interval as such:
    \begin{align*}
    &\qquad 1-\hat{\alpha}^{\text{RI}}_{n,b}(\bx)\\
    &=\frac{1}{n}\sum_{i=1}^n\mathbbm{1}\Big[\widehat \f_b^{(2)}(\bx)\in \big[\widehat{\f}_b^{(1)}(\bx)-\sqrt{2}(c^D\norm{\r_n^D})z_{1-\alpha/2}\sigma,\widehat{\f}_b^{(1)}(\bx)+\sqrt{2}(c^D\norm{\r_n^D})z_{1-\alpha/2}\sigma\big]\Big]
    \end{align*}
    And by a similar argument we can prove almost sure convergence.
\end{proof}

\subsection{Proof for Corollary \ref{cor:non-asymptotic-risk}}
\begin{proof}
    To give a PAC argument, we decompose our error into statistical error $\widehat \f_*(\bx)-\f(\bx)$ and algorithmic error $\widehat{\f}_b(\bx)-\widehat{\f}_*(\bx)$.

    From the Theorem \ref{thm:main},
    \(
    \Pr\bigl(|\widehat f_*(\bx)-f(\bx)|>\varepsilon\bigr)
          \le
          2\exp\!\bigl(-\tfrac{\varepsilon^{2}}
                             {2(c^{D,P})^{2}}\,n^{1/(d+1)}\bigr).
    \)
    Under the stated lower bound on $n$ this probability is $\le\delta/2$.
    
    To quantify the algorithmic error, we wish to make use of the Komolgorov inequality in Theorem \ref{thm:stochastic-contraction-mapping}. 
    
    For $\forall \bx \in [0,1]^d$, let $\widehat{\f}_b(\bx)$ be the stage‑$b$ predictor, let $\widehat{\f}_*(\bx)$ be the fixed point defined in Theorem \ref{thm:krrconv}, and write
    \[
    \bz_b(\bx)=\widehat \f_b(\bx)-\widehat \f_*(\bx),\qquad 
    \epsilon_b=\bz_b-\mathbb E[\bz_b\mid\mathcal F_{b-1}],
    \]
    It holds that the process $\bz_b(\bx)$ is eventually also a stochastic contraction mapping by the exact same argument we made in proving Theorem \ref{thm:krrconv}, only replacing $\bK_n$ with $\k_n(\bx)$ which still have row sum smaller than 1. And even we only have in probability guarantee of the staged mean contraction condition hold, since we are doing truncations and we can get rid off it after sufficiently large $b$, we can still check out the mean contraction condition $\sum_{b=1}^\infty(1-\lambda_b)=\infty$. Hence we invoke the Kolmogorov Inequality given in \ref{thm:stochastic-contraction-mapping}.

    Consider the decomposition:
    \[
    \Pr\!\bigl(\sup_{t\ge b}|\bz_t|>\varepsilon\bigr)
    \leq
    \Pr\bigl(|\bz_b|>\tfrac{\varepsilon}{2}\bigr)
    +
    \Pr\bigl(\sup_{t\ge b}|\bz_t|>|\bz_b|+\tfrac{\varepsilon}{2}\bigr)
    \]

    Because the SCM conditions are satisfied from $b$ onward and
    $
    \sum_{t>b}\mathbb E[\epsilon_t^{2}]\le C^{2}/b,
    $
    , where $C=2\lambda M(1+\sqrt{n})$. Theorem \ref{thm:stochastic-contraction-mapping} gives
    \[
    \Pr\bigl(\sup_{t\ge b}|\bz_t|>|\bz_b|+\tfrac{\varepsilon}{2}\bigr)
    \le
    \frac{4\cdot1\cdot (C^{2}/b)}{(\varepsilon/2)^{2}}
    \;=\;
    \frac{16C^{2}}{\varepsilon^{2}b}.
    \]

    Since $|\epsilon_k|\le C$,  
    \(\displaystyle V_b=\sum_{k=1}^{b}\E[\epsilon_k^{2}]
              \le C^{2}\pi^{2}/6<2C^{2}\).
    
    Set the martingale sum \(S_b=\sum_{k=1}^{b}\epsilon_k\).
    Because $|\bz_t|\le|\bz_{t-1}|+|\epsilon_t|$, telescoping gives
    \[
    |\bz_b|
    \;\le\;
    |\bz_{0}|+|S_b|
    \;\le\;
    M+|S_b|.
    \]
    Thus the event $\{|\bz_b|>\varepsilon/2\}$ implies
    \(
    |S_b|>(\varepsilon-2M)/2.
    \)
    
    Freedman’s maximal inequality for increments bounded by $C$ and 
    quadratic variation $\le2C^{2}$ yields
    \[
    \Pr\!\Bigl(|\bz_b|>\tfrac{\varepsilon}{2}\Bigr)
    \;\le\;
    \Pr\!\Bigl(|S_b|>\tfrac{\varepsilon-2M}{2}\Bigr)
    \;\le\;
    2\exp\!\Bigl(
          -\frac{(\varepsilon-2M)_{+}^{2}}{16\,C^{2}}
        \Bigr).
    \]
    Then the non-asymptotic algorithmic error bound:
    \[
    \Pr\!\Bigl(|\bz_b(\bx)|>\varepsilon\Bigr)
    \;\le\;
    2\exp\!\Bigl(
          -\frac{(\varepsilon-4M)_{+}^{2}}{16\,C^{2}}
        \Bigr)
    \;+\;
    \frac{16\,C^{2}}{\varepsilon^{2}b},
    \quad
    C=2\lambda M(1+\sqrt n).
    \]
    With the stated lower bound on $b$ this is $\le\delta/2$.
\end{proof}


\section{Miscellanious Lemmas}

\begin{thm}[Kolmogorov's Extension Theorem]
\label{thm:komolgorov-extension}
	Define the random variables:
	\begin{itemize}
		\item $\X =(\bx_1,\bx_2,...)\in [0,1]^{d\times \mathbb{N}}$
		\item $\vec{\epsilon}=(\epsilon_1,\epsilon_2,\cdots) \in \mathbb{R}^{\mathbb{N}}$
	\end{itemize}
	\textbf{The probability measure on $[0,1]^{d\times \mathbb{N}}$ and $\mathbb{R}^{\mathbb{N}}$ are uniquely determined by their product measures on the cylinder spaces}, reflecting the i.i.d. sampling of $\X \text{ and }\vec{\epsilon}$, such that 
	\[
	\by_i=\f(\bx_i)+\epsilon_{i}, \quad i\in \mathbb{N}
	\]
\end{thm}	

\begin{lem}
    \label{lem:fixtorandom}
    Let \( X: \Omega_1 \to S \) and \( \epsilon: \Omega_2 \to S \) be independent random variables. Assume \( \{f_n: S \times S \to \mathbb{R} \} \) is a sequence of measurable functions. Suppose that for almost surely \( x \in \Omega_1 \) and \( \epsilon \in \Omega_2 \),
    \[
    f_n(x, \epsilon) \xrightarrow{d} N(0,1).
    \]
    Then:
    \[
    f_n(X, \epsilon) \xrightarrow{d} N(0,1).
    \]
    Or equivalently 
    \[
    \lim\limits_{n \to \infty} \mathbb{P}(f_n(X,\vec{\epsilon}))=\Phi(t)
    \]
\end{lem}
\begin{proof}
    \begin{align}
        \lim_n \mathbb{P}(f_n(X, \epsilon) \leq t) 
        &= \lim_n  \int \int \mathbbm{1}_{\{f_n(x, \epsilon) \leq t\}} \, d\mu_x \, d\mu_{\epsilon} \\
        &= \lim_n  \int \mathbb{P}(f_n(x, \epsilon) \leq t) \, d\mu_x\\
        &= \int  \lim_n \mathbb{P} (f_n(x, \epsilon) \leq t) \, d\mu_x\\
        &= \int \Phi(t) \, d\mu_x = \Phi(t).
    \end{align}
    (8) is justified by the Fubini theorem since the indicator function is always non-negative. And (9) is guaranteed by Dominated Convergence Theorem since the integrand is always bounded (with assumption of it being a sub-Gaussian density of $\epsilon$).
\end{proof}

\newpage
\section{Mean Squared Error on UCI Machine Learning Repository}
\label{app:mse}
We provide further results for the MSE of our methods, benchmarks, and competitors on six additional UCI datasets. The dataset we are using in this work are from: \cite{abalone_1}, \cite{air_quality_360}, \cite{automobile_10}, \cite{estimation_of_obesity_levels_based_on_eating_habits_and_physical_condition__544}, \cite{student_performance_320}, \cite{wine_quality_186}, \cite{communities_and_crime_183}.

\begin{figure}[H]
    \centering
    \includegraphics[width=1.0\linewidth]{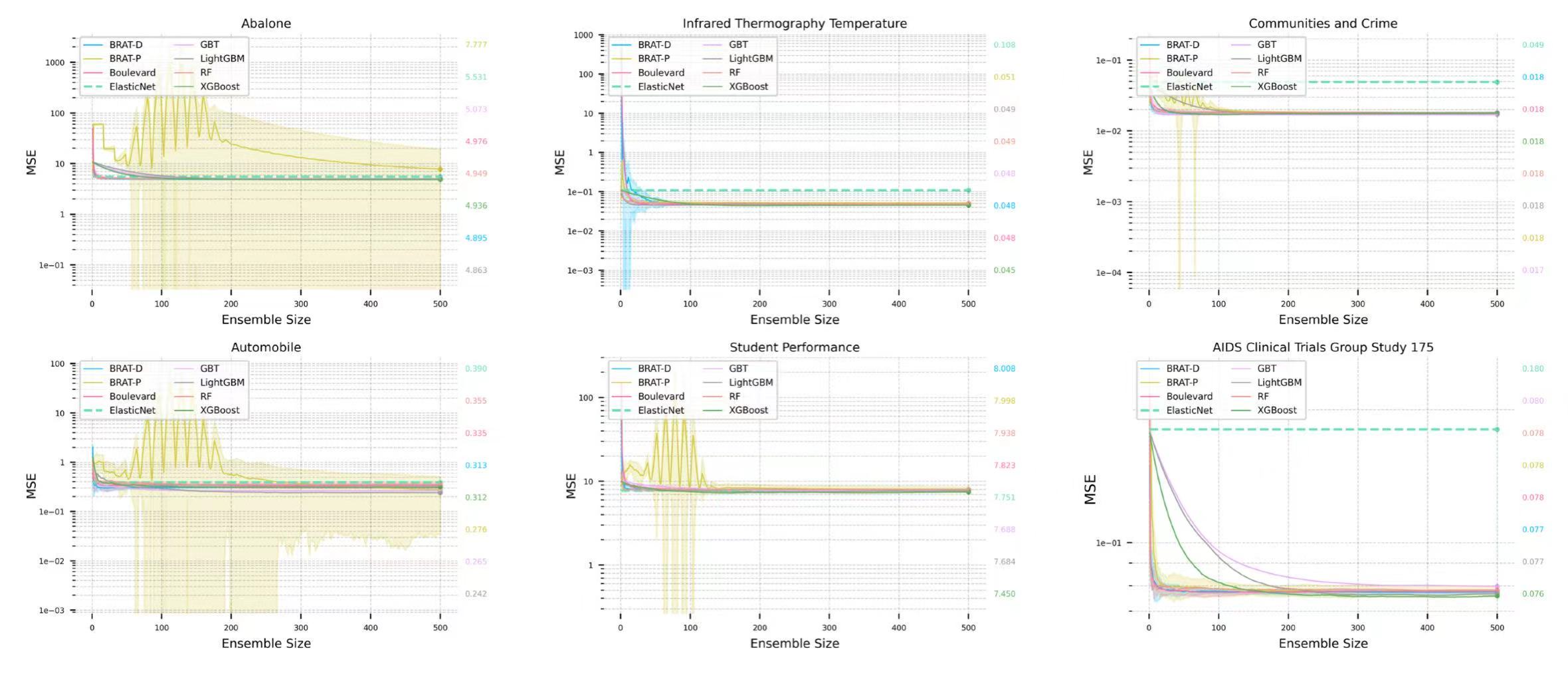}
    \caption{Mean Squared Error on UCI Machine Learning Repositories}
    \label{fig:additional-mse}
\end{figure}

The hyperparameters chosen by Optuna are listed below:
\begin{table}[H]
  \centering\small
  \caption{Abalone}
  \label{tab:abalone}
  \begin{tabular}{@{}lcccccccc@{}}
    \toprule
    Hyperparameter               & GBT    & XGBoost & LightGBM & RF     & BRAT-D & Boulevard & BRAT-P & ElasticNet \\
    \midrule
    \texttt{n\_estimators}       & 500    & 500     & 500      & 500    & 500    & 500       & 500    & —          \\
    \texttt{learning\_rate}      & 0.0213 & 0.0249  & 0.0108   & —      & 0.5476 & 0.5750    & 0.2724 & —          \\
    \texttt{max\_depth}          & 3      & 3       & 9        & 11     & 10     & 16        & 3      & —          \\
    \texttt{subsample}           & —      & 0.6332  & —        & —      & —      & —         & —      & —          \\
    \texttt{num\_leaves}         & —      & —       & 37       & —      & —      & —         & —      & —          \\
    \texttt{alpha}               & —      & —       & —        & —      & —      & —         & —      & 0.1299     \\
    \texttt{l1\_ratio}           & —      & —       & —        & —      & —      & —         & —      & 0.3753     \\
    \texttt{min\_samples\_split} & —      & —       & —        & —      & 43     & 42        & 5     & —          \\
    \texttt{subsample\_rate}     & —      & —       & —        & —      & 0.5071 & 0.5133    & 0.6266 & —          \\
    \texttt{dropout\_rate}       & —      & —       & —        & —      & 0.4766 & 0.0000    & —      & —          \\
    \texttt{n\_trees\_per\_group}& —      & —       & —        & —      & —      & —         & 16      & —          \\
    \bottomrule
  \end{tabular}
\end{table}

\begin{table}[H]
  \centering\small
  \caption{Automobile}
  \label{tab:automobile}
  \begin{tabular}{@{}lcccccccc@{}}
    \toprule
    Hyperparameter               & GBT    & XGBoost & LightGBM & RF     & BRAT-D & Boulevard & BRAT-P & ElasticNet \\
    \midrule
    \texttt{n\_estimators}       & 500    & 500     & 500      & 500    & 500    & 500       & 500    & —          \\
    \texttt{learning\_rate}      & 0.2019 & 0.2153  & 0.2341   & —      & 0.9411 & 0.1416    & 0.5671 & —          \\
    \texttt{max\_depth}          & 3      & 3       & 8        & 14     & 6      & 11        & 3     & —          \\
    \texttt{subsample}           & —      & 0.9850  & —        & —      & —      & —         & —      & —          \\
    \texttt{num\_leaves}         & —      & —       & 37       & —      & —      & —         & —      & —          \\
    \texttt{alpha}               & —      & —       & —        & —      & —      & —         & —      & 0.3634     \\
    \texttt{l1\_ratio}           & —      & —       & —        & —      & —      & —         & —      & 0.1386     \\
    \texttt{min\_samples\_split} & —      & —       & —        & —      & 23     & 9         & 30     & —          \\
    \texttt{subsample\_rate}     & —      & —       & —        & —      & 0.9368 & 0.9899    & 0.7344 & —          \\
    \texttt{dropout\_rate}       & —      & —       & —        & —      & 0.2145 & 0.0000    & —      & —          \\
    \texttt{n\_trees\_per\_group}& —      & —       & —        & —      & —      & —         & 16     & —          \\
    \bottomrule
  \end{tabular}
\end{table}

\begin{table}[H]
  \centering\small
  \caption{Communities and Crime}
  \label{tab:com-crime}
  \begin{tabular}{@{}lcccccccc@{}}
    \toprule
    Hyperparameter               & GBT    & XGBoost & LightGBM & RF     & BRAT-D & Boulevard & BRAT-P & ElasticNet \\
    \midrule
    \texttt{n\_estimators}       & 500    & 500     & 500      & 500    & 500    & 500       & 500    & —          \\
    \texttt{learning\_rate}      & 0.0224 & 0.0379  & 0.0215   & —      & 0.6904 & 0.5772    & 0.3578 & —          \\
    \texttt{max\_depth}          & 6      & 3       & 5        & 16     & 15     & 16        & 4     & —          \\
    \texttt{subsample}           & —      & 0.5375  & —        & —      & —      & —         & —      & —          \\
    \texttt{num\_leaves}         & —      & —       & 20       & —      & —      & —         & —      & —          \\
    \texttt{alpha}               & —      & —       & —        & —      & —      & —         & —      & 3.8079     \\
    \texttt{l1\_ratio}           & —      & —       & —        & —      & —      & —         & —      & 0.9556     \\
    \texttt{min\_samples\_split} & —      & —       & —        & —      & 35     & 49        & 5     & —          \\
    \texttt{subsample\_rate}     & —      & —       & —        & —      & 0.7385 & 0.7209    & 0.5148 & —          \\
    \texttt{dropout\_rate}       & —      & —       & —        & —      & 0.1903 & 0.0000    & —      & —          \\
    \texttt{n\_trees\_per\_group}& —      & —       & —        & —      & —      & —         & 11      & —          \\
    \bottomrule
  \end{tabular}
\end{table}

\begin{table}[H]
  \centering\small
  \caption{Wine Quality}
  \label{tab:wine-quality}
  \begin{tabular}{@{}lcccccccc@{}}
    \toprule
    Hyperparameter               & GBT    & XGBoost & LightGBM & RF     & BRAT-D & Boulevard & BRAT-P & ElasticNet \\
    \midrule
    \texttt{n\_estimators}       & 500    & 500     & 500      & 500    & 500    & 500       & 500    & —          \\
    \texttt{learning\_rate}      & 0.0731 & 0.1139  & 0.1186   & —      & 0.8396 & 0.8905    & 0.8200 & —          \\
    \texttt{max\_depth}          & 10     & 10      & 16       & 20     & 15     & 15        & 16     & —          \\
    \texttt{subsample}           & —      & 0.6980  & —        & —      & —      & —         & —      & —          \\
    \texttt{num\_leaves}         & —      & —       & 42       & —      & —      & —         & —      & —          \\
    \texttt{alpha}               & —      & —       & —        & —      & —      & —         & —      & 3.8079     \\
    \texttt{l1\_ratio}           & —      & —       & —        & —      & —      & —         & —      & 0.9556     \\
    \texttt{min\_samples\_split} & —      & —       & —        & —      & 4      & 2         & 13     & —          \\
    \texttt{subsample\_rate}     & —      & —       & —        & —      & 0.7317 & 0.7431    & 0.6669 & —          \\
    \texttt{dropout\_rate}       & —      & —       & —        & —      & 0.2762 & 0.0000    & —      & —          \\
    \texttt{n\_trees\_per\_group}& —      & —       & —        & —      & —      & —         & 10     & —          \\
    \bottomrule
  \end{tabular}
\end{table}

\begin{table}[H]
  \centering\small
  \caption{Student Performance}
  \label{tab:student}
  \begin{tabular}{@{}lcccccccc@{}}
    \toprule
    Hyperparameter               & GBT    & XGBoost & LightGBM & RF     & BRAT-D & Boulevard & BRAT-P & ElasticNet \\
    \midrule
    \texttt{n\_estimators}       & 500    & 500     & 500      & 500    & 500    & 500       & 500    & —          \\
    \texttt{learning\_rate}      & 0.0177 & 0.1836  & 0.0218   & —      & 0.9256 & 0.8698    & 0.9824 & —          \\
    \texttt{max\_depth}          & 3      & 5       & 5        & 11     & 3      & 9         & 3     & —          \\
    \texttt{subsample}           & —      & 0.5780  & —        & —      & —      & —         & —      & —          \\
    \texttt{num\_leaves}         & —      & —       & 38       & —      & —      & —         & —      & —          \\
    \texttt{alpha}               & —      & —       & —        & —      & —      & —         & —      & 0.4855     \\
    \texttt{l1\_ratio}           & —      & —       & —        & —      & —      & —         & —      & 0.1373     \\
    \texttt{min\_samples\_split} & —      & —       & —        & —      & 39     & 32        & 48     & —          \\
    \texttt{subsample\_rate}     & —      & —       & —        & —      & 0.8418 & 0.5249    & 0.5964 & —          \\
    \texttt{dropout\_rate}       & —      & —       & —        & —      & 0.1230 & 0.0000    & —      & —          \\
    \texttt{n\_trees\_per\_group}& —      & —       & —        & —      & —      & —         & 13     & —          \\
    \bottomrule
  \end{tabular}
\end{table}

\begin{table}[H]
  \centering\small
  \caption{Obesity Level}
  \label{tab:obesity}
  \begin{tabular}{@{}lcccccccc@{}}
    \toprule
    Hyperparameter               & GBT    & XGBoost & LightGBM & RF     & BRAT-D & Boulevard & BRAT-P & ElasticNet \\
    \midrule
    \texttt{n\_estimators}       & 500    & 500     & 500      & 500    & 500    & 500       & 500    & —          \\
    \texttt{learning\_rate}      & 0.1905 & 0.0916  & 0.0834   & —      & 0.2487 & 0.5382    & 0.3643 & —          \\
    \texttt{max\_depth}          & 7      & 7       & 16       & 16     & 16     & 16        & 11     & —          \\
    \texttt{subsample}           & —      & 0.8121  & —        & —      & —      & —         & —      & —          \\
    \texttt{num\_leaves}         & —      & —       & 25       & —      & —      & —         & —      & —          \\
    \texttt{alpha}               & —      & —       & —        & —      & —      & —         & —      & 0.2182     \\
    \texttt{l1\_ratio}           & —      & —       & —        & —      & —      & —         & —      & 0.1200     \\
    \texttt{min\_samples\_split} & —      & —       & —        & —      & 3      & 3         & 4     & —          \\
    \texttt{subsample\_rate}     & —      & —       & —        & —      & 0.6466 & 0.7153    & 0.5595 & —          \\
    \texttt{dropout\_rate}       & —      & —       & —        & —      & 0.8079 & 0.0000    & —      & —          \\
    \texttt{n\_trees\_per\_group}& —      & —       & —        & —      & —      & —         & 4     & —          \\
    \bottomrule
  \end{tabular}
\end{table}

\begin{table}[H]
  \centering\small
  \caption{AIDS Clinical Trials Group Study 175}
  \label{tab:aids}
  \begin{tabular}{@{}lcccccccc@{}}
    \toprule
    Hyperparameter               & GBT    & XGBoost & LightGBM & RF     & BRAT-D & Boulevard & BRAT-P & ElasticNet \\
    \midrule
    \texttt{n\_estimators}       & 500    & 500     & 500      & 500    & 500    & 500       & 500    & —          \\
    \texttt{learning\_rate}      & 0.0177 & 0.0226  & 0.0108   & —      & 0.8519 & 0.9137    & 0.3010 & —          \\
    \texttt{max\_depth}          & 3      & 3       & 9        & 5      & 4      & 4         & 4     & —          \\
    \texttt{subsample}           & —      & 0.7530  & —        & —      & —      & —         & —      & —          \\
    \texttt{num\_leaves}         & —      & —       & 37       & —      & —      & —         & —      & —          \\
    \texttt{alpha}               & —      & —       & —        & —      & —      & —         & —      & 3.8079     \\
    \texttt{l1\_ratio}           & —      & —       & —        & —      & —      & —         & —      & 0.9556     \\
    \texttt{min\_samples\_split} & —      & —       & —        & —      & 12     & 27        & 46     & —          \\
    \texttt{subsample\_rate}     & —      & —       & —        & —      & 0.5335 & 0.7374    & 0.5342 & —          \\
    \texttt{dropout\_rate}       & —      & —       & —        & —      & 0.3421 & 0.0000    & —      & —          \\
    \texttt{n\_trees\_per\_group}& —      & —       & —        & —      & —      & —         & 3      & —          \\
    \bottomrule
  \end{tabular}
\end{table}

\begin{table}[H]
  \centering\small
  \caption{Infrared Thermography Temperature}
  \label{tab:temperature}
  \begin{tabular}{@{}lcccccccc@{}}
    \toprule
    Hyperparameter               & GBT    & XGBoost & LightGBM & RF     & BRAT-D & Boulevard & BRAT-P & ElasticNet \\
    \midrule
    \texttt{n\_estimators}       & 500    & 500     & 500      & 500    & 500    & 500       & 500    & —          \\
    \texttt{learning\_rate}      & 0.0289 & 0.0167  & 0.1836   & —      & 0.7616 & 0.6558    & 0.9864 & —          \\
    \texttt{max\_depth}          & 5      & 3       & 1        & 20     & 3      & 3         & 3     & —          \\
    \texttt{subsample}           & —      & 0.6988  & —        & —      & —      & —         & —      & —          \\
    \texttt{num\_leaves}         & —      & —       & 24       & —      & —      & —         & —      & —          \\
    \texttt{alpha}               & —      & —       & —        & —      & —      & —         & —      & 3.8079     \\
    \texttt{l1\_ratio}           & —      & —       & —        & —      & —      & —         & —      & 0.9556     \\
    \texttt{min\_samples\_split} & —      & —       & —        & —      & 4      & 4         & 9     & —          \\
    \texttt{subsample\_rate}     & —      & —       & —        & —      & 0.5055 & 0.5438    & 0.6371 & —          \\
    \texttt{dropout\_rate}       & —      & —       & —        & —      & 0.2395 & 0.0000    & —      & —          \\
    \texttt{n\_trees\_per\_group}& —      & —       & —        & —      & —      & —         & 2     & —          \\
    \bottomrule
  \end{tabular}
\end{table}
The learning rate for BRAT-P is only valid in the first vanilla boosting round. After that, all trees are fixed with learning rate 1.

\newpage
\section{Rainclouds for Interval Coverage}
\label{app:rain-clouds}
Following the discussion in Section \ref{sec:coverage-experiments}, we present more raincloud plots below:
\begin{figure}[H]
    \centering
    \includegraphics[width=1.0\linewidth]{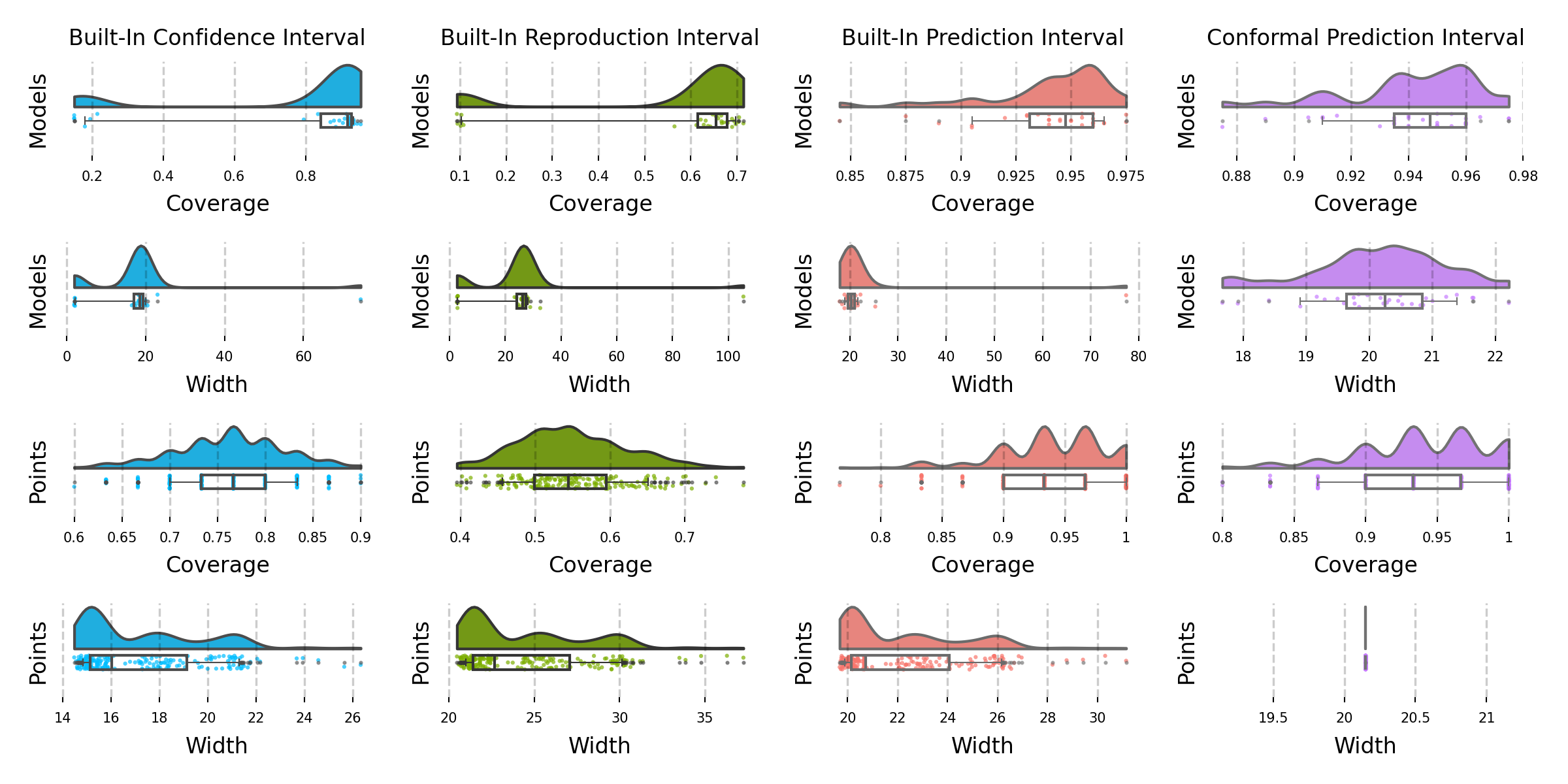}
    \caption{Moderate coverage, where there is neither overfitting nor underfitting.}
    \subcaption{Test size: 200; Model Replications: 30; Ensemble Size: 200; Learning Rate: 0.6; Subsampling Rate: 0.8; Dropout Rate: 0.3; Maximum Depth: 4; Nystrom subsampling rate: 0.1.}
    \label{fig:moderate-fit}
\end{figure}

\begin{figure}[H]
    \centering
    \includegraphics[width=1.0\linewidth]{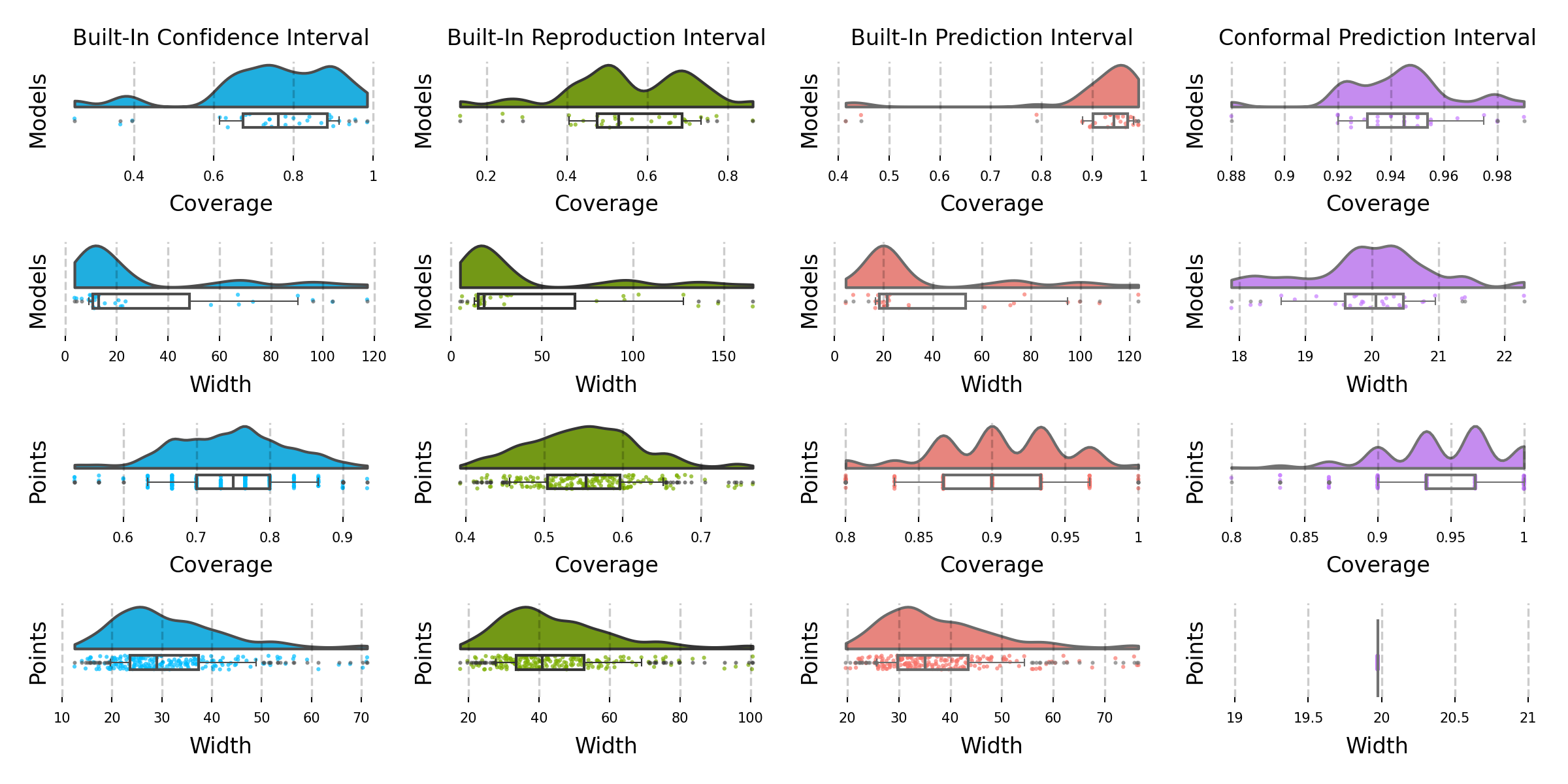}
    \caption{Wide coverage, when there is overfitting.}
    \subcaption{Test size: 200; Model Replications: 30; Ensemble Size: 200; Learning Rate: 0.3; Subsampling Rate: 0.8; Dropout Rate: 0.3; Maximum Depth: 6; Nystrom subsampling rate: 0.1.}
    \label{fig:over-fit}
\end{figure}

\begin{figure}[H]
    \centering
    \includegraphics[width=1.0\linewidth]{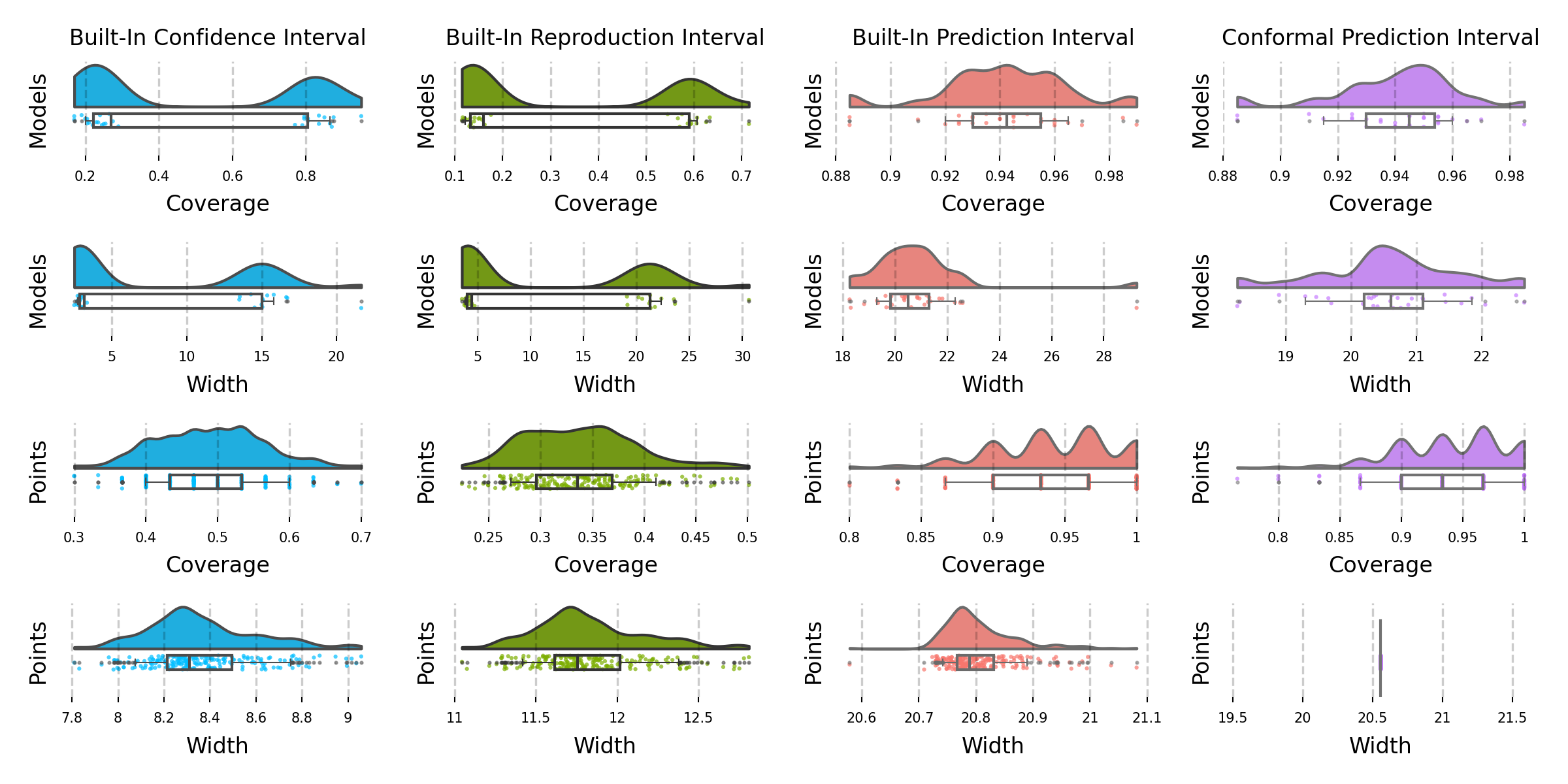}
    \caption{Narrow coverage, when there is underfitting.}
    \subcaption{Test size: 200; Model Replications: 30; Ensemble Size: 200; Learning Rate: 1.0; Subsampling Rate: 0.8; Dropout Rate: 0.9; Maximum Depth: 4; Nystrom subsampling rate: 0.1.}
    \label{fig:under-fit}
\end{figure}

\newpage
\section{Visualizations of Interval Coverage on a 1D Signal}
\label{app:1d-interval}
As an extension of Section \ref{sec:coverage-experiments}, we display more visualizations of our intervals on a 1D interval.
\begin{figure}[H]
    \centering
    \includegraphics[width=1.0\linewidth]{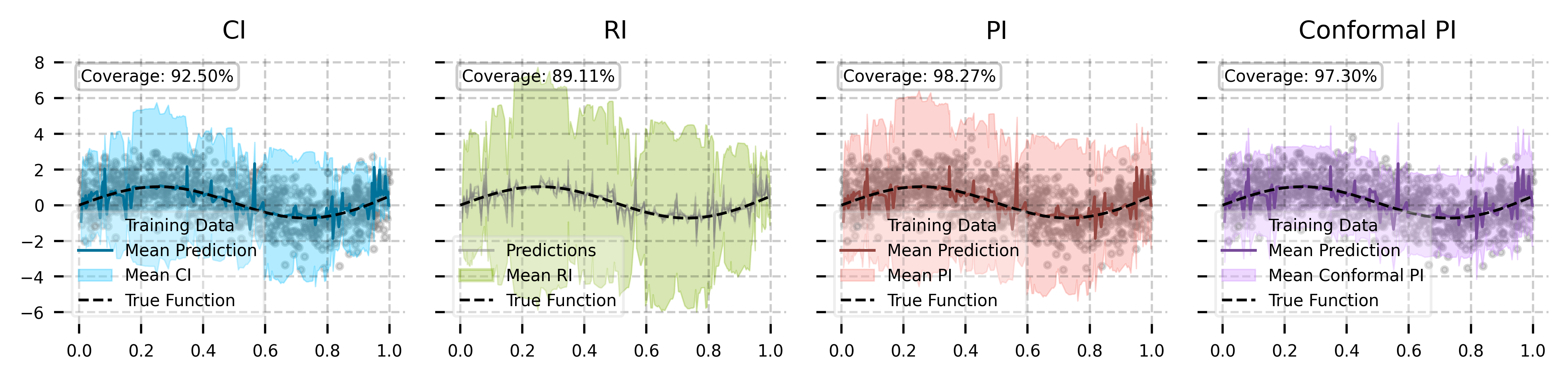}
    \caption{Moderate coverage, neither overfitting nor underfitting.}
    \subcaption{Ensemble Size: 300; Learning Rate: 1.0; Maximum Depth: 8; Subsampling Rate: 0.9; Dropout rate: 0.6}
    \label{fig:1d-moderate}
\end{figure}

\begin{figure}[H]
    \centering
    \includegraphics[width=1.0\linewidth]{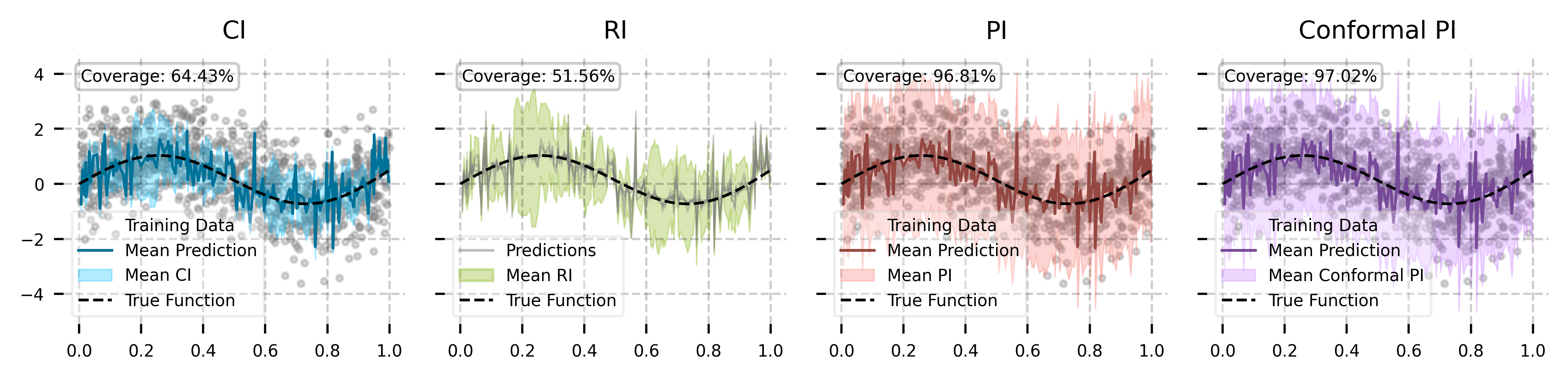}
    \caption{Low coverage, overfitting.}
    \subcaption{Ensemble Size: 100; Learning Rate: 0.3; Maximum Depth: 12; Subsampling Rate: 0.6; Dropout rate: 0.3}
    \label{fig:1d-over-fit}
\end{figure}

\begin{figure}[H]
    \centering
    \includegraphics[width=1.0\linewidth]{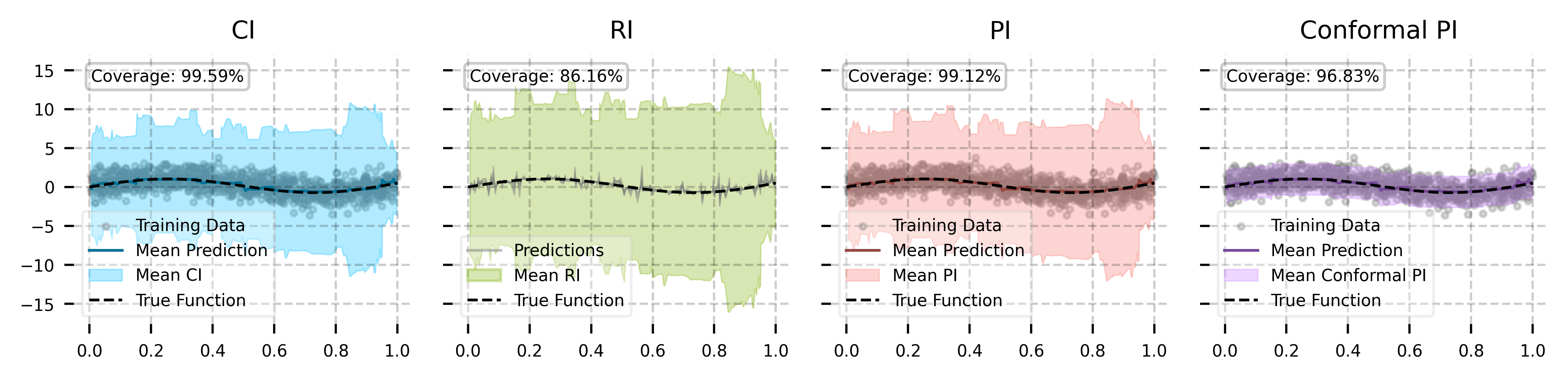}
    \caption{High coverage, underfitting.}
    \subcaption{Ensemble Size: 100; Learning Rate: 0.3; Maximum Depth: 4; Subsampling Rate: 0.6; Dropout rate: 0.6}
    \label{fig:1d-under-fit}
\end{figure}

\newpage
\section*{NeurIPS Paper Checklist}


\begin{enumerate}

\item {\bf Claims}
    \item[] Question: Do the main claims made in the abstract and introduction accurately reflect the paper's contributions and scope?
    \item[] Answer: \answerYes{} 
    \item[] Justification: The abstract and introduction state that we (1) propose both random-dropout and structured-dropout boosting algorithms, (2) derive their asymptotic normality via a CLT, and (3) establish finite-sample statistical rates. These claims correspond directly to algorithms definitions and CLT statement in Section \ref{sec:algorithms} (proof in Appendix \ref{app:random-clt}). Statistical rates are given in Section \ref{sec:statistics} (proofs in Appendix \ref{app:statistics}).  

\item {\bf Limitations}
    \item[] Question: Does the paper discuss the limitations of the work performed by the authors?
    \item[] Answer: \answerYes{} 
    \item[] Justification: We explicitly acknowledge in Section \ref{sec:discussions} that our theoretical guarantees rest on strong assumptions (structure–value isolation, non-adaptivity, and tree regularity) which may not hold in practice, and we outline the need to relax these; we also note that our methods are currently limited to regression and discuss challenges in extending to classification, survival analysis, and structured outputs.

\item {\bf Theory assumptions and proofs}
    \item[] Question: For each theoretical result, does the paper provide the full set of assumptions and a complete (and correct) proof?
    \item[] Answer: \answerYes{} 
    \item[] Justification: The paper mainly introduced two modified boosting algorithms: Algorithm \ref{alg:random-dropout} and Algorithm \ref{alg:structured-dropout}. In Section \ref{sec:setup} and \ref{sec:theory} we introduced all assumptions needed for our following results to hold. Both algorithms are shown to converge to a fixed point in Appendix \ref{app:finite-sample-convergence}. And the corresponding derivations of CLTs are provided in Appendix \ref{app:random-clt}.

    \item {\bf Experimental result reproducibility}
    \item[] Question: Does the paper fully disclose all the information needed to reproduce the main experimental results of the paper to the extent that it affects the main claims and/or conclusions of the paper (regardless of whether the code and data are provided or not)?
    \item[] Answer: \answerYes{} 
    \item[] Justification: All parameters used to produce the presented results are either provided in Section \ref{sec:coverage-experiments}, or provided in the appendix.

\item {\bf Open access to data and code}
    \item[] Question: Does the paper provide open access to the data and code, with sufficient instructions to faithfully reproduce the main experimental results, as described in supplemental material?
    \item[] Answer: \answerYes{} 
    \item[] Justification: We include all data‐processing scripts, model training and evaluation code, and a README with exact commands and environment specifications in the supplementary materials. All UCI datasets used are public and instructions for download and preprocessing are provided.

\item {\bf Experimental setting/details}
    \item[] Question: Does the paper specify all the training and test details (e.g., data splits, hyperparameters, how they were chosen, type of optimizer, etc.) necessary to understand the results?
    \item[] Answer: \answerYes{} 
    \item[] Justification:  We describe our train/test splits, evaluation protocol, and Optuna‐based hyperparameter optimization in Section \ref{sec:experiments}, with full model settings tabulated in Appendix \ref{app:mse}.

\item {\bf Experiment statistical significance}
    \item[] Question: Does the paper report error bars suitably and correctly defined or other appropriate information about the statistical significance of the experiments?
    \item[] Answer: \answerYes{} 
    \item[] Justification: In MSE races, we run models for 5 replications and fill the mean MSE trajectories $+/- 2\widehat{s.d.}$. In coverage rate simulations we provide kernel density estimations.

\item {\bf Experiments compute resources}
    \item[] Question: For each experiment, does the paper provide sufficient information on the computer resources (type of compute workers, memory, time of execution) needed to reproduce the experiments?
    \item[] Answer: \answerYes{} 
    \item[] Justification: All experiments were run on an Apple MacBook Pro (2022, Apple M2, 8 GB RAM); individual runs took under 2 hours and the full suite under 6 hours. These modest requirements mean the results can be reproduced on any modern laptop.
    
\item {\bf Code of ethics}
    \item[] Question: Does the research conducted in the paper conform, in every respect, with the NeurIPS Code of Ethics \url{https://neurips.cc/public/EthicsGuidelines}?
    \item[] Answer: \answerYes{} 
    \item[] Justification: All authors have reviewed the NeurIPS Code of Ethics and will preserve anonymity.

\item {\bf Broader impacts}
    \item[] Question: Does the paper discuss both potential positive societal impacts and negative societal impacts of the work performed?
    \item[] Answer: \answerNA{}{} 
    \item[] Justification: This work develops statistical inference tools for gradient boosting in tabular regression—a foundational methodological advance without a specific application domain or direct deployment context. As such, there is no clear pathway to societal harms or benefits tied exclusively to our contributions.
    
\item {\bf Safeguards}
    \item[] Question: Does the paper describe safeguards that have been put in place for responsible release of data or models that have a high risk for misuse (e.g., pretrained language models, image generators, or scraped datasets)?
    \item[] Answer: \answerNA{} 
    \item[] Justification: Our work focuses on statistical inference for gradient boosting on tabular regression tasks, involving no pretrained networks, scraped media, or high-risk generative models, and thus does not pose significant misuse risks.


\item {\bf Licenses for existing assets}
    \item[] Question: Are the creators or original owners of assets (e.g., code, data, models), used in the paper, properly credited and are the license and terms of use explicitly mentioned and properly respected?
    \item[] Answer: \answerYes{} 
    \item[] Justification: We cite each UCI dataset by name and reference (all are publicly available under the UCI Machine Learning Repository terms), and we reference the original papers for all libraries and methods used (e.g., XGBoost, LightGBM).

\item {\bf New assets}
    \item[] Question: Are new assets introduced in the paper well documented and is the documentation provided alongside the assets?
    \item[] Answer: \answerNA{} 
    \item[] Justification: We do not introduce or release any new datasets, codebases, or pre-trained models—our experiments use publicly available UCI datasets and standard libraries.


\item {\bf Crowdsourcing and research with human subjects}
    \item[] Question: For crowdsourcing experiments and research with human subjects, does the paper include the full text of instructions given to participants and screenshots, if applicable, as well as details about compensation (if any)? 
    \item[] Answer: \answerNA{}{} 
    \item[] Justification: All experiments are done on sythetic dataset or existing dataset from UCI repository.

\item {\bf Institutional review board (IRB) approvals or equivalent for research with human subjects}
    \item[] Question: Does the paper describe potential risks incurred by study participants, whether such risks were disclosed to the subjects, and whether Institutional Review Board (IRB) approvals (or an equivalent approval/review based on the requirements of your country or institution) were obtained?
    \item[] Answer: \answerNA{} 
    \item[] Justification: Our paper does not involve crowdsourcing nor research with human subjects.

\item {\bf Declaration of LLM usage}
    \item[] Question: Does the paper describe the usage of LLMs if it is an important, original, or non-standard component of the core methods in this research? Note that if the LLM is used only for writing, editing, or formatting purposes and does not impact the core methodology, scientific rigorousness, or originality of the research, declaration is not required.
    \item[] Answer: \answerNA{} 
    \item[] Justification: All core methodological components were designed and implemented by the authors. Any use of LLMs was limited to writing, editing, or formatting and did not influence the scientific content or originality of the work.

\end{enumerate}

\end{document}

%% file: macros.tex
\usepackage{amsmath,bbm,bm}
\usepackage{amsfonts}
\usepackage{amsthm}
\usepackage{mathtools}

\usepackage{algorithm}
\usepackage{algorithmic}

\newtheorem{thm}{Theorem}
\newtheorem{lem}{Lemma}

\newtheorem{cor}{Corollary}

\newtheorem{aspt}{Assumption}

\newtheorem{defn}{Definition}

\usepackage{xcolor}
\newcount\comments  
\newcommand{\genComment}[2]{\ifnum\comments=1{\textcolor{#1}{\textsf{\footnotesize #2}}}\fi}

\def\eqref#1{equation~\ref{#1}}

\def\1{\bm{1}}
\newcommand{\train}{\mathcal{D}}

\newcommand{\test}{\mathcal{D_{\mathrm{test}}}}

\DeclareMathAlphabet{\mathsfit}{\encodingdefault}{\sfdefault}{m}{sl}
\SetMathAlphabet{\mathsfit}{bold}{\encodingdefault}{\sfdefault}{bx}{n}

\def\gG{{\mathcal{G}}}

\def\gL{{\mathcal{L}}}

\def\gN{{\mathcal{N}}}

\def\gQ{{\mathcal{Q}}}

\def\gS{{\mathcal{S}}}

\def\prob{{\mathbb{P}}}

\def\R{{\mathbb{R}}}

\newcommand{\E}{\mathbb{E}}

\def\bx{{\mathbf{x}}}
\def\bX{{\mathbf{X}}}
\def\by{{\mathbf{y}}}

\def\bz{{\mathbf{z}}}

\def\bS{{\mathbf{S}}}
\def\bI{{\mathbf{I}}}
\def\f{{\bm{f}}}
\def\d{{\bm{d}}}
\def\g{{\bm{g}}}
\def\t{{\bm{t}}}
\def\r{{\bm{r}}}
\def\s{{\bm{s}}}

\def\br{{\mathbf{r}}}
\def\bk{{\mathbf{k}}}
\def\k{{\bm{k}}}
\def\bK{{\mathbf{K}}}

\def\balpha{{\boldsymbol{\alpha}}}
\def\bLambda{{\boldsymbol{\Lambda}}}
\def\bSigma{{\boldsymbol{\Sigma}}}
\def\bkappa{{\boldsymbol{\kappa}}}
\def\bXi{{\boldsymbol{\Xi}}}

\def\test{\text{test}}
\def\train{\text{train}}